\documentclass{article}

\usepackage{microtype}
\usepackage{graphicx}
\usepackage{subfigure}
\usepackage{booktabs} 

\usepackage{hyperref}

\usepackage[accepted]{icml2024}

\usepackage{amsmath}
\usepackage{amssymb}
\usepackage{mathtools}
\usepackage{amsthm}

\usepackage{graphicx}
\usepackage{caption}
\usepackage{subcaption}
\usepackage{float}
\usepackage{pifont}
\usepackage[inline]{enumitem}
\usepackage{wrapfig}
\usepackage{centernot}
\usepackage{adjustbox}
\usepackage{bm}
\usepackage{multirow}
\usepackage[capitalize,noabbrev]{cleveref}

\theoremstyle{plain}
\newtheorem{theorem}{Theorem}
\newtheorem{proposition}{Proposition}
\newtheorem{lemma}{Lemma}

\theoremstyle{definition}

\newtheorem{remark}{Remark}

\usepackage[textsize=tiny]{todonotes}

\icmltitlerunning{Graph Adversarial Diffusion Convolution}

\begin{document}

\twocolumn[
\icmltitle{Graph Adversarial Diffusion Convolution}



\icmlsetsymbol{equal}{*}

\begin{icmlauthorlist}
\icmlauthor{Songtao Liu}{psu}
\icmlauthor{Jinghui Chen}{psu}
\icmlauthor{Tianfan Fu}{rpi}
\icmlauthor{Lu Lin}{psu}
\icmlauthor{Marinka Zitnik}{harvard}
\icmlauthor{Dinghao Wu}{psu}
\end{icmlauthorlist}

\icmlaffiliation{psu}{The Pennsylvania State University}
\icmlaffiliation{rpi}{Rensselaer Polytechnic Institute}
\icmlaffiliation{harvard}{Harvard University}

\icmlcorrespondingauthor{Songtao Liu}{skl5761@psu.edu}

\icmlkeywords{Machine Learning, ICML}

\vskip 0.3in
]



\printAffiliationsAndNotice{}  

\begin{abstract}
This paper introduces a min-max optimization formulation for the Graph Signal Denoising (GSD) problem. In this formulation, we first maximize the second term of GSD by introducing perturbations to the graph structure based on Laplacian distance and then minimize the overall loss of the GSD. By solving the min-max optimization problem, we derive a new variant of the Graph Diffusion Convolution (GDC) architecture, called Graph Adversarial Diffusion Convolution (GADC). GADC differs from GDC by incorporating an additional term that enhances robustness against adversarial attacks on the graph structure and noise in node features. Moreover, GADC improves the performance of GDC on heterophilic graphs. Extensive experiments demonstrate the effectiveness of GADC across various datasets. Code is available at \url{https://github.com/SongtaoLiu0823/GADC}.

\end{abstract}

\section{Introduction}
Graph Neural Networks (GNNs)~\citep{kipf2017semi,hamilton2017inductive,velivckovic2018graph} have become a popular approach for graph-based tasks due to their powerful ability to learn node representations. They have demonstrated remarkable performance in various tasks, including traffic prediction~\citep{guo2019attention}, drug discovery~\citep{dai2019retrosynthesis}, and recommendation system~\citep{ying2018graph}. The core principle of GNNs is the message-passing operation, which aggregates node features from neighboring nodes, thereby enhancing the smoothness of the learned node representations. As a result, GNN models produce predictions based on both the features of individual nodes and those of their immediate neighbors.

Graph Diffusion Convolution (GDC)~\citep{gasteiger2019diffusion}, a specialized GNN architecture, aggregates information from higher-order neighbors through generalized graph diffusion. Various GDC architectures~\citep{li2019label,zhu2021simple,liu2021elastic,yang2021graph,zhao2021adaptive,jia2022unifying} are derived from the Graph Signal Denoising (GSD) problem, formulated as:
\begin{equation}
\label{graph signal denoising}
\underset{\mathbf{F}}{\arg \min } \ \mathcal{L}(\mathbf{F}):=\|\mathbf{F}-\mathbf{X}\|_F^2+\lambda \operatorname{tr}\left(\mathbf{F}^{\top} \tilde{\mathbf{L}} \mathbf{F}\right),
\end{equation}
where $\mathbf{X}=\mathbf{X}^*+\bm{\Upsilon}$ is the observed noisy feature matrix, $\bm{\Upsilon} \in \mathbb{R}^{n \times d}$ is the noise matrix, $\mathbf{X}^*$ is the clean feature matrix, and $\tilde{\mathbf{L}} \in \mathbb{R}^{n \times n}$ is the normalized graph Laplacian matrix. A major strength of the GDC architecture is its ability to aggregate information from a larger set of neighbors, enhancing $\ell_2$-based graph smoothing~\citep{liu2021elastic}.

Despite significant advancements in GDC architectures, they rely heavily on the Laplacian matrix to derive the graph diffusion matrix. This dependency can be problematic when the graph structure is disrupted by adversarial attacks~\citep{dai2018adversarial,zugner2018adversarial,zugner2019adversarial,jin2020graph}. Such attacks can cause GDC to aggregate harmful information, disrupting node representations. Additionally, the limited number of neighbors restricts GDC's ability to effectively filter out large noise in node features within some graphs~\citep{liu2021graph}. These challenges raise a crucial question: \emph{Can we develop a versatile GDC architecture that addresses these issues effectively?} In this work, we provide a positive solution to this question by reformulating the GSD problem. Based on this reformulation, we design a new GDC architecture that mitigates the negative impact of adversarial attacks on the graph structure, noise in node features, and inconsistent edges on heterophilic graphs.

Inspired by the saddle point formulation of adversarial training~\citep{madry2018towards}, we propose a min-max variant of the GSD problem, called the Adversarial Graph Signal Denoising (AGSD) problem: 
\begin{equation}
\label{intro:agsd}
      \underset{\mathbf{F}}{\arg\min } \  q(\mathbf{F}) := \left[\left\|\mathbf{F}-\mathbf{X}\right\|_{F}^{2}+\lambda \cdot \max_{\mathbf{L}^{\prime}} \  \operatorname{tr}\left(\mathbf{F}^{\top} \mathbf{L}^{\prime} \mathbf{F}\right)\right],
\end{equation}
where $\mathbf{L}^{\prime}$ is the added perturbation by the adversary. In the inner optimization problem, we maximize the second term of the AGSD. Then, we minimize the loss function $q(\mathbf{F})$ in the outer optimization problem. 

Unlike adversarial training, which requires projected gradient descent (PGD) to solve the inner maximization problem, our formulation has a closed-form solution for the inner optimization problem. Utilizing this closed-form solution, we solve the outer minimization problem to derive a new GDC architecture, Graph Adversarial Diffusion Convolution (GADC). This architecture introduces an additional term compared to GDC, resulting in more adaptive GDC variants that are resilient to adversarial perturbations in the edges and noise in node features. Moreover, our GADC enhances the performance of GDC on heterophilic graphs, where nodes from different classes are linked. Extensive experimental results across various datasets validate the effectiveness of our proposed algorithm. Additionally, our novel AGSD problem can inspire the development of advanced and reliable GNN architectures. 

\section{Preliminaries}
\paragraph{Notations.} Let $\mathcal{G}=(\mathcal{V}, \mathcal{E})$ represent a undirected graph, where $\mathcal{V}$ is the set of vertices $\left\{v_{1}, \cdots, v_{n}\right\}$ with $|\mathcal{V}|=n$ and $\mathcal{E}$ is the set of edges. The adjacency matrix is defined as $\mathbf{A} \in \{0,1\}^{n \times n}$, and $\mathbf{A}_{i,j}=1$ if and only if $\left(v_{i}, v_{j}\right) \in \mathcal{E}$. Let $\mathcal{N}_{i}=\{v_j | \mathbf{A}_{i,j} = 1\}$ denote the neighborhood of node $v_i$ and $\mathbf{D}$ denote the diagonal degree matrix, where $\mathbf{D}_{i,i}=\sum_{j=1}^{n} \mathbf{A}_{i,j}$. The normalized Laplacian matrix of a graph is defined as $\mathbf{L}=\operatorname{Laplacian}(\mathbf{A})=\mathbf{I}_n-\mathbf{D}^{-1 / 2} \mathbf{A D}^{-1 / 2}$, where $\mathbf{I}_n$ is an $n \times n$ identity matrix. The feature matrix is denoted as $\mathbf{X} \in \mathbb{R}^{n \times d}$. $\operatorname{tr}\left(\cdot\right)$ denotes the trace of the matrix. 

\paragraph{Spectral Graph Convolution.} Spectral convolution involves multiplying a signal $\mathbf{x}$ by a Fourier domain filter $g_\phi$, parameterized by coefficients $\phi \in \mathbb{R}^n$. This is expressed as:
\begin{equation}
g_\phi(\mathbf{L}) \star \mathbf{x}=\mathbf{U} g_\phi^*(\boldsymbol{\Lambda}) \mathbf{U}^{\top} \mathbf{x},
\end{equation}
where $g_\phi$ can be approximated by a truncated expansion. A common approach is to use Chebyshev polynomials up to the $K$-th order, resulting in the following approximation:
\begin{equation}
g_\phi^*(\boldsymbol{\Lambda}) \approx \sum_{k=0}^K \phi_k T_k(\tilde{\boldsymbol{\Lambda}}).
\end{equation}
\paragraph{Graph Convolution Network (GCN).} GCNs approximate spectral graph convolution using first-order Chebyshev polynomials. By setting $K=1$ and $\phi_0=-\phi_1$, and approximating $\lambda_n \approx 2$, the convolution becomes $g_\phi(\mathbf{L}) \star \mathbf{x}=\left(\mathbf{I}_n+\mathbf{D}^{-1 / 2} \mathbf{A} \mathbf{D}^{-1 / 2}\right) \mathbf{x}$. Introducing self-loops and renormalization trick modifies this to $\tilde{\bm{\mathcal{A}}}=\tilde{\mathbf{D}}^{-1 / 2} \tilde{\mathbf{A}} \tilde{\mathbf{D}}^{-1 / 2}$, where $\tilde{\mathbf{A}}=\mathbf{A}+\mathbf{I}_n$ and $\tilde{\mathbf{D}}=\mathbf{D}+\mathbf{I}_n$. The GCN layer is then defined as:
\begin{equation}
\mathbf{H}^{(l+1)}=\sigma\left(\tilde{\bm{\mathcal{A}}} \mathbf{H}^{(l)} \boldsymbol{\Theta}^{(l)}\right).
\end{equation}
Here, $\sigma$ is the activation function and $\boldsymbol{\Theta}^{(l)}$ are the trainable parameters in the $l$-th layer. 
\paragraph{Adversarial Training.}
\citet{madry2018towards} study the adversarial robustness of neural networks from the perspective of robust optimization as follows:
\begin{equation}
\underset{\Theta}{\arg\min } \ \rho(\Theta):=\mathbb{E}_{(x, y) \sim \mathcal{D}}\left[\max _{\delta \in \mathcal{S}} L(\Theta, x+\delta, y)\right],
\end{equation}
where $\Theta$ is the set of model parameters, $x$ is the input sample, $y$ is the label, $\mathcal{D}$ is the data distribution, $\delta$ is the perturbation, $\mathcal{S}$ is the perturbation space, $L$ is the loss function, and $\mathbb{E}$ is the empirical risk. The inner maximization problem attacks the neural network, while the outer minimization problem finds model parameters to make it robust against these adversarial attacks. In this work, we introduce a min-max formulation for the GSD problem, known as AGSD.
\paragraph{Graph Diffusion Convolution (GDC).}
Graph diffusion~\citep{gasteiger2019diffusion} is defined via the diffusion matrix: 
\begin{equation}
\label{eq:diffusion matrix}
\mathbf{S}=\sum_{k=0}^{\infty} \theta_k \mathbf{T}^k,
\end{equation} 
where $\theta_k$ are the weighting coefficients and $\mathbf{T}$ is the transition matrix. Existing GNN architectures such as APPNP~\citep{klicpera2019predict}, GLP~\citep{li2019label}, S$^2$GC~\citep{zhu2021simple}, and AirGNN~\citep{liu2021graph} can be viewed as variations of GDC. In this work, we introduce a new GDC architecture based on our proposed AGSD.

\section{Graph Adversarial Diffusion Convolution}
In this section, we first introduce the min-max variant of the graph signal denoising (GSD) problem in Section~\ref{ssec:AGSD}. Then, we discuss the closed-form solution for the inner maximization problem in Section~\ref{ssec:closed-form solution}. In Section~\ref{ssec:gadnn}, we propose graph adversarial diffusion convolution (GADC) by solving the outer minimization problem. Finally, Sections~\ref{adversarial attack}, \ref{sec:noisy graphs}, and \ref{sec:heterophilic graph} introduce four adaptive GADC variants designed to enhance robustness against adversarial attacks on the graph structure and noise in node features, as well as improve performance on heterophilic graphs.

\subsection{Adversarial Graph Signal Denoising Problem} 
\subsubsection{Problem Formulation}
\label{ssec:AGSD}
Adversarial training has been recently explored in the robust optimization. \citet{madry2018towards} use a saddle point (min-max) formulation to incorporate protection against adversarial attacks into neural networks. In the inner maximization problem, the adversary uses projected gradient descent (PGD) to identify the worst-case adversarial perturbations that maximize the loss. The outer minimization problem seeks to find model parameters that minimize the adversarial loss generated by the inner attack problem. 

Inspired by the saddle point (min-max) formulation in adversarial training, we introduce the Adversarial Graph Signal Denoising (AGSD) problem as shown in Eq.~\eqref{intro:agsd}. In the AGSD problem, the first term ensures the solution is close to the observed data, while the second term involves a graph Laplacian matrix. Defining perturbations on the graph structure to maximize the second term in the AGSD problem is less straightforward than traditional adversarial attacks. Inspired by the recent work~\citep{lin2022graph} that uses spectral distance to deploy graph structural attacks, we leverage Laplacian distance to introduce perturbations. 
\paragraph{Laplacian Distance.} \citet{lin2022graph} define the spectral distance as the changes in the eigenvalues of the graph Laplacian matrix, formulated as follows:
\begin{equation}
    \mathcal{D}_{\text {spectral }}=\left\|g_\phi^*(\boldsymbol{\Lambda})-g_\phi^*\left(\boldsymbol{\Lambda}^{\prime}\right)\right\|_2,
\end{equation}
where $\boldsymbol{\Lambda}$ and $\boldsymbol{\Lambda}^{\prime}$ denote the eigenvalues of the normalized graph Laplacian matrix for the original graph $\mathcal{G}$ and the disrupted graph $\mathcal{G}^{\prime}$. The idea behind using the spectral distance for graph structural attacks is that perturbing the graph to maximize this distance induces the most harmful perturbation to the graph filters~\citep{chang2021not} and significantly disrupts node embeddings. To make the attack simple and effective, we directly add perturbations to the graph filter. This allows us to define the Laplacian distance.
\begin{equation}
    \mathcal{D}_{\text {Laplacian }}=\left\|g_\phi(\tilde{\mathbf{L}})-g_\phi\left(\mathbf{L}^{\prime}\right)\right\|_F
   =\left\|\mathbf{L}^{\prime}-\tilde{\mathbf{L}}\right\|_{F},
\end{equation}
where $\tilde{\mathbf{L}}=\mathbf{I}_n-\tilde{\mathbf{D}}^{-1 / 2} \tilde{\mathbf{A}} \tilde{\mathbf{D}}^{-1 / 2}$ denotes the normalized graph Laplacian matrix.
\paragraph{AGSD.} Based on the Laplacian distance, we introduce our AGSD as follows:
\begin{equation}
\label{adv gsd}
\begin{gathered}
      \underset{\mathbf{F}}{\arg\min } \  q(\mathbf{F}) := \left[\left\|\mathbf{F}-\mathbf{X}\right\|_{F}^{2}+\lambda \cdot \max_{\mathbf{L}^{\prime}} \  \operatorname{tr}\left(\mathbf{F}^{\top} \mathbf{L}^{\prime} \mathbf{F}\right)\right] \\  \operatorname{s.t.}  \ \left\|\mathbf{L}^{\prime}-\tilde{\mathbf{L}}\right\|_{F} \leq \varepsilon.
      \end{gathered}
\end{equation}
In the inner optimization problem, a hypothetical adversary generates perturbations based on the Laplacian distance to create a modified Laplacian matrix, $\mathbf{L}^{\prime}$, which aims to maximize the second term of the AGSD. Then, the outer minimization problem seeks to find $\mathbf{F}$ that minimizes the overall loss function $q(\mathbf{F})$.
\begin{figure}[t]
 \centering
\includegraphics[width=0.5\textwidth]{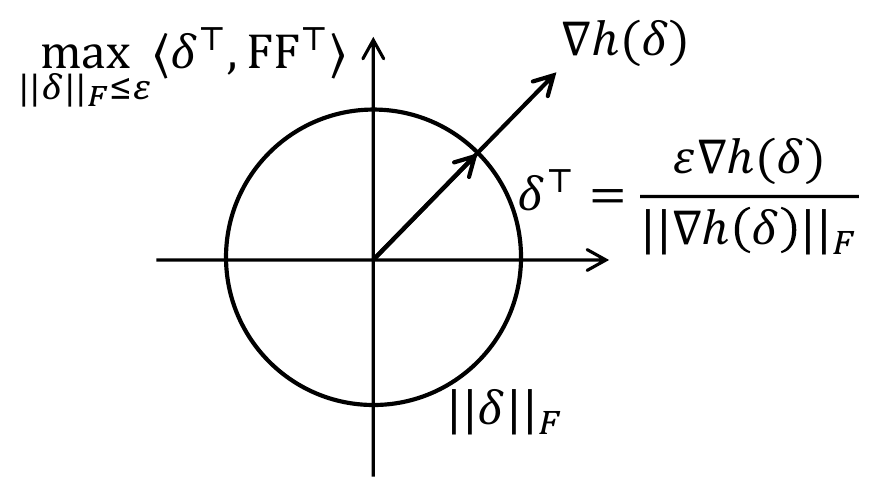}
\caption{Illustration of the inner maximization problem. The loss function reaches the largest value when the direction of $\bm{\delta}^{\top}$ is the same as $\nabla h(\bm{\delta})$.}
\label{figure:adv inner problem}
\end{figure}
\subsubsection{Closed-form Solution of the Inner Maximization Problem}
\label{ssec:closed-form solution}
The min-max formulation in Eq.\eqref{adv gsd} introduces a more complex GSD problem. In contrast to adversarial training~\citep{madry2018towards}, where the inner maximization problem is solved using PGD before solving the outer minimization problem at each training step, our inner maximization problem is a quadratic optimization problem. This removes the need for computationally expensive adversarial perturbations at each step. Instead, we can directly identify the largest perturbation. 

Let us denote the perturbations as $\bm{\delta}$, and the perturbed Laplacian matrix as $\mathbf{L}'=\tilde{\mathbf{L}} + \bm{\delta}$. Due to the quadratic nature of our inner maximization problem, we can reformulate it as follows:
\begin{equation}
\begin{gathered}
    \max_{\mathbf{L}^{\prime}} \operatorname{tr}\left(\mathbf{F}^{\top} \mathbf{L}^{\prime} \mathbf{F}\right) =\langle\tilde{\mathbf{L}}^{\top}, \mathbf{F}\mathbf{F}^{\top}\rangle + \max_{\bm{\delta}}\langle\bm{\delta}^{\top}, \mathbf{F}\mathbf{F}^{\top}\rangle \\ \operatorname{s.t.} \ \left\|\bm{\delta}\right\|_{F} \leq \varepsilon.
    \end{gathered}
\end{equation}
We denote $h(\bm{\delta})=\langle\bm{\delta}^{\top}, \mathbf{F}\mathbf{F}^{\top}\rangle$. $h(\bm{\delta})$ reaches the largest value when $\bm{\delta}^{\top}$ has the same direction with the gradient of $h(\bm{\delta})$, i.e. $\bm{\delta}^{\top}=\frac{\varepsilon\nabla h(\bm{\delta})}{\left\|\nabla h(\bm{\delta})\right\|_{F}}=\frac{\varepsilon\mathbf{F}\mathbf{F}^{\top}}{\left\|\mathbf{F}\mathbf{F}^{\top}\right\|_{F}}$. An illustration is provided in Figure~\ref{figure:adv inner problem}. Plugging this solution into Eq.~\eqref{adv gsd}, we can rewrite the outer optimization problem as follows:
\begin{equation}
\label{eq:p(h)}
\begin{gathered}
\underset{\mathbf{F}}{\arg\min } \  q(\mathbf{F}), \\ 
         q(\mathbf{F})=\left[\left\|\mathbf{F}-\mathbf{X}\right\|_{F}^{2}+\lambda \operatorname{tr}\left(\mathbf{F}^{\top} \tilde{\mathbf{L}} \mathbf{F}\right)+\lambda\varepsilon\operatorname{tr}\frac{\mathbf{F}^{\top}\mathbf{F}\mathbf{F}^{\top}\mathbf{F}}{\left\|\mathbf{F}\mathbf{F}^{\top}\right\|_{F}}\right],
        \end{gathered}
\end{equation}
where we introduce an extra loss term $\lambda\varepsilon\operatorname{tr}\frac{\mathbf{F}^{\top}\mathbf{F}\mathbf{F}^{\top}\mathbf{F}}{\left\|\mathbf{F}\mathbf{F}^{\top}\right\|_{F}}$ compared with the GSD.
\subsubsection{Modified Transition Matrix Derived by Outer Minimization Problem} 
\label{ssec:gadnn}
In this section, we derive our graph adversarial diffusion convolution. By taking the gradient of Eq.~\eqref{eq:p(h)} to zero, we get the solution of the outer optimization problem as:
\begin{equation}
\label{eq:advF}
    \mathbf{F} = \left(\mathbf{I}+\lambda\tilde{\mathbf{L}}+\lambda\frac{\varepsilon\mathbf{F}\mathbf{F}^{\top}}{\left\|\mathbf{F}\mathbf{F}^{\top}\right\|_{F}}\right)^{-1}\mathbf{X}.
\end{equation}
Computing Eq.~(\ref{eq:advF}) directly involves a matrix inverse operation, resulting in a complexity of $\mathcal{O}\left(n^3\right)$. This high computational cost can be prohibitively expensive for large graphs. Inspired by the closed-form solution of Personalized PageRank (PPR) kernel~\citep{brin1998pagerank,gasteiger2019diffusion} (graph diffusion), we can solve Eq.~\eqref{eq:advF} via graph diffusion (Eq.~\eqref{eq:diffusion matrix}) and obtain
\begin{equation}
\label{eq:graph adversarial diffusion inverse}
    \mathbf{F} = \frac{1}{\lambda+1} \sum_{k=0}^K\left[\frac{\lambda}{\lambda+1} \left(\tilde{\bm{\mathcal{A}}}-\frac{\varepsilon\mathbf{F}\mathbf{F}^{\top}}{\left\|\mathbf{F}\mathbf{F}^{\top}\right\|_{F}}\right)\right]^k\mathbf{X},
\end{equation}
where $\mathbf{T}=\tilde{\mathbf{D}}^{-1 / 2} \tilde{\mathbf{A}} \tilde{\mathbf{D}}^{-1 / 2}-\frac{\varepsilon\mathbf{F}\mathbf{F}^{\top}}{\left\|\mathbf{F}\mathbf{F}^{\top}\right\|_{F}}=\tilde{\bm{\mathcal{A}}}-\frac{\varepsilon\mathbf{F}\mathbf{F}^{\top}}{\left\|\mathbf{F}\mathbf{F}^{\top}\right\|_{F}}$ represents the transition matrix, and $\alpha =\frac{1}{\lambda+1}$ denotes the coefficient.
\paragraph{Approximate Solution of Eq.~\eqref{eq:graph adversarial diffusion inverse}.} Although we approximate the inverse matrix with graph diffusion to avoid $\mathcal{O}(n^3)$ complexity, the matrix series Eq.~\eqref{eq:graph adversarial diffusion inverse} still involves high powers of the matrix $\mathbf{F}$. Obviously, it is impossible to get an analytical solution of $\mathbf{F}$ for this equation. Current methods exploit iterative algorithms such as Newton's method~\citep{more1982newton} to solve this high-order nonlinear system of equations, but the computation is also very expensive. Therefore, there is no suitable solution that strictly follows the equation while also avoiding large computational complexity. We need a more efficient algorithm to solve Eq.~\eqref{eq:graph adversarial diffusion inverse}. We observe that the first term of AGSD requires $\mathbf{F}$ to be close to $\mathbf{X}$, thus a natural idea is to replace $\mathbf{F}$ with $\mathbf{X}$ in the right side of Eq.~\eqref{eq:graph adversarial diffusion inverse} to reduce the computational complexity and improve the scalability for large graphs. Thus, we now formally propose our GADC as follows:
\begin{equation}
\label{eq:graph adversarial diffusion}
    \mathbf{S} = \frac{1}{\lambda+1}\sum_{k=0}^{K}\left[\frac{\lambda}{\lambda+1}\mathbf{T}\right]^{k},
\end{equation}
where $\mathbf{T}=\tilde{\bm{\mathcal{A}}}-\frac{\varepsilon\mathbf{X}\mathbf{X}^{\top}}{\left\|\mathbf{X}\mathbf{X}^{\top}\right\|_{F}}$. Note that the additional term $\frac{\varepsilon\mathbf{X}\mathbf{X}^{\top}}{\left\|\mathbf{X}\mathbf{X}^{\top}\right\|_{F}}$ is a unique component introduced by the closed-form solution of the inner maximization problem of AGSD. This term doesn't appear in existing GDC architectures derived from the GSD problem. We will use this term to produce more adaptive architectures, as demonstrated in the following sections.
\paragraph{Error Analysis.} We plug our obtained $\mathbf{F}$ into Eq.~\eqref{eq:graph adversarial diffusion inverse} to compute the error of this approximate solution for the equation on the Cora~\citep{sen2008collective} dataset. The error matrix is computed by subtracting the left side from the right side: $\mathbf{F}-\frac{1}{\lambda+1} \sum_{k=0}^K\left[\frac{\lambda}{\lambda+1} \mathbf{T}\right]^k \mathbf{X}$, where $\mathbf{T}=\tilde{\bm{\mathcal{A}}}-\frac{\varepsilon\mathbf{X}\mathbf{X}^{\top}}{\left\|\mathbf{X}\mathbf{X}^{\top}\right\|_{F}}$. The norm of the error matrix for our solution is 4.5. As a comparison, we generate a random $\mathbf{F}$ from a Gaussian distribution with mean=1, std=1. If we use this random $\mathbf{F}$, the norm of the error matrix is 2780.6. Comparing the norm of the error matrix between our solution (4.5) and a random solution (2780.6), we can conclude that replacing $\mathbf{F}$ with $\mathbf{X}$ indeed achieves an accurate approximation, while greatly reducing the computational complexity.
\paragraph{Scalability.} To leverage the additional term across various graphs, we enhance the scalability of the transition matrix by providing the following options: computing the normalized or unnormalized inner product of feature vectors $\left(\mathbf{X}_i, \mathbf{X}_j \mid j \in \mathcal{N}_i\right)$ between adjacent neighbors, or between every pair of nodes, similar to the masked/unmasked attention mechanism in GAT. Therefore, our modified transition matrix can be formulated as follows:
\begin{equation}
\label{eq:options}
\begin{split}
 &\mathbf{T}=\tilde{\bm{\mathcal{A}}}-\varepsilon\mathbf{\Phi}, \ \operatorname{where} \\ 
 \operatorname{(\uppercase\expandafter{\romannumeral1}):} \ &\Phi_{ij}:= \begin{cases} \frac{\mathbf{X}_i \mathbf{X}_j^{\top}}{\left\|\mathbf{X}_i\right\|_2\left\|\mathbf{X}_j\right\|_2}, & \text {if} \left(v_i, v_j\right) \in \mathcal{E} \\ 0, & \text {otherwise},\end{cases}; \\
 \operatorname{(\uppercase\expandafter{\romannumeral2}):} \ &\Phi_{ij}:=\frac{\mathbf{X}_i \mathbf{X}_j^{\top}}{\left\|\mathbf{X}\mathbf{X}^{\top}\right\|_{F}};\\
 \operatorname{(\uppercase\expandafter{\romannumeral3}):} \ &\Phi_{ij}:= \begin{cases} \frac{\mathbf{X}_i \mathbf{X}_j^{\top}}{\left\|\mathbf{X}\mathbf{X}^{\top}\right\|_{F}}, & \text {if} \left(v_i, v_j\right) \in \mathcal{E} \\ 0, & \text {otherwise},\end{cases}.
 \end{split}
\end{equation}
For the first option, we normalize the weights by using the product of the norms of the features of connected neighbors, which measures the similarity between these neighboring nodes' features. We will derive the fourth option from the first option to defend against adversarial attacks on the graph structure. The second option improves graph connectivity, filtering out large noise in graphs. Furthermore, we discover that even without explicitly incorporating additional graph connectivity (the third option), reassigning edge weights through the additional term improves denoising performance on large graphs. Moreover, using the first option to penalize the weights of connected neighbors enhances performance on heterophilic graphs. We will discuss the details of these options in the subsequent sections. 
\paragraph{Connection to APPNP.} If we take the limit $k \rightarrow \infty$ of power iterations in APPNP~\citep{klicpera2019predict}, APPNP converges to $\mathbf{Z}^{(\infty)}=\alpha\left(\mathbf{I}-(1-\alpha) \tilde{\bm{\mathcal{A}}}\right)^{-1} \mathbf{Z}^{(0)}$, similar to Eq.~\eqref{eq:advF}. Our proposed GADC introduces an additional term $\varepsilon \boldsymbol{\Phi}$ in the transition matrix. This additional term makes GADC more adaptable compared to APPNP. As we will demonstrate in the upcoming sections, this adaptability enhances resilience in various scenarios, including adversarial attacks on the graph structure, noise in node features, and inconsistent edges on heterophilic graphs.

\subsection{Defending Against Graph Adversarial Attacks}
\label{adversarial attack}
Recent studies~\citep{zugner2018adversarial,zugner2019adversarial} have shown that GNNs are vulnerable to adversarial attacks on the graph structure. These attacks often involve adding harmful edges or removing informative ones to manipulate the information flow, thus preventing GNNs from aggregating valuable information and disrupting node representations. GDC relies on the graph structure to derive the diffusion matrix needed for smoothing node representations. However, when the graph structure is disrupted by adversarial attacks, it becomes unreliable, and we cannot use the graph Laplacian matrix to aggregate information from neighbors. Inspired by GNNGuard~\citep{zhang2020gnnguard}, which assigns higher weights to edges connecting nodes with similar features while pruning edges between unrelated nodes, we employ the additional term to achieve a similar effect. The rationale is that on homophily graphs, nodes within the same class have similar features, while nodes from different classes exhibit dissimilar features. Therefore, we use the first option in Eq.~\eqref{eq:options} and assign a very large value to the additional term in the modified transition matrix, allowing this term to determine the edge weights:
\begin{equation}
  \mathbf{T}=\tilde{\bm{\mathcal{A}}}-\lim_{\varepsilon\rightarrow \infty}\varepsilon\mathbf{\Phi}.
\end{equation}
However, this approach may result in a numerical explosion problem. To address this issue, we introduce a trick in which we compute the modified transition matrix based on the additional term with $\mathbf{T}=-\lim_{\varepsilon\rightarrow \infty}\frac{1}{\varepsilon}\left(\tilde{\bm{\mathcal{A}}}-\varepsilon\mathbf{\Phi}\right)$. As such, we employ the cosine value to evaluate the similarity among connected neighbors and recalculate the weights of the transition matrix as follows: 
\begin{equation}
\label{eq:fourth}
  \operatorname{(\uppercase\expandafter{\romannumeral4}):} \  \mathbf{T}_{ij} = \left(\mathbf{X}_i \odot \mathbf{X}_j\right) /\left(\left\|\mathbf{X}_i\right\|_2\left\|\mathbf{X}_j\right\|_2\right) \ \operatorname{s.t.}\ \mathbf{A}_{ij}=1,
\end{equation}
where $\mathbf{A}$ is the disrupted adjacency matrix by adversarial attacks and $\odot$ denotes the inner product. By implementing this trick, we can propose the fourth option to reconstruct the informative adjacency matrix using the additional term, thereby restoring the beneficial information flow during the aggregation process.
\subsection{Tackling Large Noise in Node Features}
\label{sec:noisy graphs}
In the study by \citet{liu2021graph}, feature aggregation in GNNs is suggested to function as a low-pass filter, smoothing node features across neighborhoods and filtering out node feature noise~\citep{nt2019revisiting, zhao2019pairnorm}. However, on graphs with a limited number of neighbors, GDC's ability to effectively filter out large noise in node features is constrained. To address this, we provide a theoretical analysis to understand the effect and introduce two simple and effective GADC (II, III) methods to handle noisy node features in graphs. 
\subsubsection{Convergence Analysis for the Aggregated Noisy Matrix}
Consider applying GDC to a noisy node feature matrix, represented as the product of $\mathbf{S}$ and $\bm{\Upsilon}$:
\begin{equation}
\mathbf{S}\bm{\Upsilon},
\end{equation}
where $\bm{\Upsilon}$ denotes the noise matrix. Intuitively, if the matrix norm $\|\mathbf{S}\bm{\Upsilon}\|_{F}$ can converge to a small value, the denoising effect is achieved. Consider the noise utilized in~\citep{zhou2021graph,chen2021graph,zhang2022graphless} follows Gaussian distribution, which is also covered by sub-Gaussian variable. Therefore, the noise matrix has the following property. Based on this property, we provide the analysis.
\begin{proposition}[\textbf{Noise Property}]
Each entry of the noise matrix $\bm{\Upsilon}$, i.e., $[\bm{\Upsilon}]_{ij}$ is i.i.d sub-Gaussian random variable with variance $\sigma$ and mean $\mu=0$, i.e.,
\begin{equation}
    \mathbb{E}\left[e^{\lambda([\bm{\Upsilon}]_{ij}-\mu)}\right] \leq e^{\sigma^{2} \lambda^{2} / 2} \quad \text { for all } \lambda \in \mathbb{R}.
\end{equation}
\end{proposition}
\paragraph{Higher-order Graph Connectivity Factor.}
Intuitively, $\mathbf{S}$ captures not only the connectivity of the graph structure (represented by $\widetilde{\bm{\mathcal{A}}}$), but also the higher-order connectivity (represented by $\widetilde{\bm{\mathcal{A}}}^2, \widetilde{\bm{\mathcal{A}}}^3, \ldots, \widetilde{\bm{\mathcal{A}}}^K$). As we will discuss later, greater higher-order graph connectivity can accelerate the convergence of the noise matrix. To formally quantify higher-order graph connectivity, we provide the following definition: 
\begin{equation}
\label{evenness factor}
    \tau = \max_i \tau_i, \ \ \text{ where } \tau_i = {n\sum_{j=1}^{n}\left[\mathbf{S}\right]_{ij}^2}\Bigg/{\left(\left[\mathbf{S}\right]_{ij}\right)^2}. 
\end{equation}
\begin{remark} 
\emph{Here, we provide some intuition on why Eq.~(\ref{evenness factor}) represents higher-order graph connectivity. Note that each element in $\mathbf{S}$ is non-negative, and each row sum satisfies~\footnote{This result is obtained by using $\widetilde{\bm{\mathcal{A}}}=\widetilde{\mathbf{D}}^{-1}\widetilde{\mathbf{A}}$ for ease of theoretical analysis, while in experiments we use the more common $\widetilde{\bm{\mathcal{A}}}=\widetilde{\mathbf{D}}^{-\frac{1}{2}}\widetilde{\mathbf{A}} \widetilde{\mathbf{D}}^{-\frac{1}{2}}$ as suggested in GCN. The proof can be found in Appendix~\ref{appendix:row sum}.}
\begin{equation}
\label{row sum}\sum_{j=1}^{n}\left[\mathbf{S}\right]_{ij} 
    =1-\left(\frac{\lambda}{\lambda+1}\right)^{K+1}=\beta.
\end{equation}
Based on Eq.~\eqref{row sum}, the sum of squares of elements in each row satisfies:
\begin{equation}
\label{range}
   {\beta^2}\bigg/{n} \leq \sum_{j=1}^{n}\left[\mathbf{S}\right]_{ij}^2 \leq \beta^2.
\end{equation}
When a graph has high connectivity, meaning the elements in row $i$ of $\bm{S}$ are more uniformly distributed, Eq.~\eqref{range} reaches its lower bound. Conversely, if a graph is poorly connected and only one element in row $i$ is greater than zero, Eq.~\eqref{range} reaches its upper bound. Therefore, the value of $\tau \in [1, n]$ is determined as follows: when the higher-order graph connectivity factor is large, $\tau \to 1$, and when the graph is less connected, $\tau \to n$.}
\end{remark}
\begin{theorem}[\textbf{Upper Bound}]
\label{upper bound of aggregated noised matirx}
Suppose we choose $t=2\tau\left(4\log n+\log{2d}\right)/n$. Then, with a high probability of $1-1/d$, we have 
\begin{equation}
    \left\|\mathbf{S}\bm{\Upsilon}\right\|_{F}^{2}\leq \frac{2\tau\left(1-\left(\frac{\lambda}{\lambda+1}\right)^{K+1}\right)^2\sigma^2\left(4\log n+\log{2d}\right)}{n}.
\end{equation}
\end{theorem}
Proof can be found in Appendix~\ref{appendix:upper bound with hoeffding}. Theorem~\ref{upper bound of aggregated noised matirx} implies that the norm of the aggregated noise matrix $\mathbf{S}\bm{\Upsilon}$ is bounded by three terms: the number of nodes of a graph $n$, the expansion order $K$, the higher-order graph connectivity factor $\tau$. We provide the intuition behind Theorem~\ref{upper bound of aggregated noised matirx} by viewing $\mathbf{S}\bm{\Upsilon}$ as sampling from the noise probability distribution. Each row of $\mathbf{S}$ represents the weighted average of these noise samples. According to the law of large numbers, each row of $\mathbf{S}\bm{\Upsilon}$ will converge to the expected value ($\mu=0$) of the noise distribution, thus reducing the non-predictive stochasticity of the weighted average noise. The denoising effect of $\mathbf{S}\bm{\Upsilon}$ depends on the number of samples (i.e., $K$ ) and the weights $\mathbf{S}_{i j}$, both of which are influenced by the depth of GNNs and the graph structure. We provide an illustration in Appendix~\ref{appendix:illustration}.

\subsubsection{Architecture}
For low node-degree graphs, it is impossible to increase the node degree; however, introducing the additional term $\frac{\varepsilon\mathbf{X}\mathbf{X}^{\top}}{\left\|\mathbf{X}\mathbf{X}^{\top}\right\|_{F}}$ can add additional graph connectivity, thereby decreasing the higher-order graph connectivity factor. Thus, we employ the second option in Eq.~\eqref{eq:options} for low node-degree noisy graphs. Furthermore, considering the noise comes from a Gaussian distribution, the magnitude of the added noise values can be unpredictable. Nonetheless, the term $\left\|\mathbf{X}\mathbf{X}^{\top}\right\|_{F}$ can penalize the weights when noise values are excessively large, thereby mitigating the adverse effect of the noise. Note that the weights of edges also determine the factor as shown in Eq.~\eqref{evenness factor}. For high node-degree graphs, we can reassign the weights of edges through the additional term to decrease the factor as shown in the third option. In the experiment section, we will demonstrate the effectiveness of our GADC (II, III) in dealing with noise in node features for both types of graphs.

\subsection{Improving Performance on Heterophilic Graphs}
\label{sec:heterophilic graph}
The aggregation scheme in message-passing-based GNNs is generally considered harmful for learning node representations on heterophilic graphs. To alleviate the negative impact of inconsistent edges on heterophilic graphs, we employ the first option in Eq.~\eqref{eq:options} to reduce the weights of the Laplacian matrix coefficients. Thus, the additional term, $\varepsilon\mathbf{\Phi}$, can obstruct the graph smoothing process. In our subsequent ablation study, we will demonstrate our GADC (I) can improve performance on heterophilic graphs. 
\subsection{Decoupling Feature Aggregation from Downstream Training}
GADC is formulated as Eq.~\eqref{eq:graph adversarial diffusion}, with four options presented in Eq.~\eqref{eq:options} and Eq.~\eqref{eq:fourth}. Depending on the scenario, we first select the appropriate option and then calculate $\Phi_{i j}$ using the feature matrix $\mathbf{X}$. This allows us to obtain the graph diffusion matrix $\mathbf{S}$. We then multiply the graph diffusion matrix by the feature matrix $\mathbf{X}$ to get the aggregated features $\mathbf{S}\mathbf{X}$. Finally, we use the aggregated features as input to train a linear model or an MLP. This approach effectively decouples feature aggregation from downstream training. The overall process is illustrated in Algorithm~\ref{alg:framework}.
\paragraph{Differences from APPNP and GNNGuard.} For the fourth option, our algorithm is a pre-computation method that uses the additional term to reconstruct the graph structure once. After reconstructing the adjacency matrix, we use it directly to aggregate node features. GNNGuard evaluates the importance of neighbors using the cosine similarity of node embeddings for each layer and normalizes the similarity. Therefore, the features used for computing similarity differ between GADC (I) and the higher layers ($\geq 2$) of GNNGuard. Additionally, our GADC differs from APPNP. APPNP performs a nonlinear transformation on the feature matrix and then uses the graph diffusion matrix to aggregate node features. In contrast, we introduce an additional term to create more adaptable GADC architectures.
\begin{algorithm}[t]
\caption{Graph Adversarial Diffusion Convolution} 
\label{alg:framework}
\begin{algorithmic}[1]
\STATE {\bfseries Input:} Adjacency matrix $\mathbf{A}$, feature matrix $\mathbf{X}$ 
\STATE {\bfseries Output:} Prediction $\mathbf{Z}$
\STATE Compute the transition matrix $\mathbf{T}$ using Eq.~\eqref{eq:options} based on the selected option.
\STATE Compute the graph adversarial diffusion matrix $\mathbf{S}$ via Eq.~\eqref{eq:graph adversarial diffusion}.
\STATE Compute the aggregated feature matrix $\mathbf{F}=\mathbf{S}\mathbf{X}$.
\WHILE{not convergence}
\STATE Compute the output of the forward model (Linear or MLP) based on the input $\mathbf{F}$: $\mathbf{Z}=f_{\Theta}(\mathbf{F})$.
\STATE Compute the supervised loss function $\mathcal{L}_{s}$.
\STATE Update the parameters $\Theta$ via gradient descent:
$\Theta = \Theta - \eta \nabla_\Theta (\mathcal{L}_{s})$.
\ENDWHILE
\STATE Predict via $\mathbf{Z} = f_{\Theta}(\mathbf{F})$.
\end{algorithmic}
\end{algorithm}

\section{Experiments}
\label{sec:experiments}
\begin{table*}[t]
\caption{Summary of test accuracy (\%) results from 100 runs under non-adaptive adversarial attacks.}
\vspace*{-0.2in}
\label{tab:metaattack}
\begin{center}
 \scriptsize
 \resizebox{\linewidth}{!}{
 \setlength{\tabcolsep}{4pt}
\begin{tabular}{ccccccccccc}
\toprule
\multirow{2}{*}{Ptb Rate (\%)} & \multicolumn{3}{c}{Cora} & \multicolumn{3}{c}{Citeseer} & \multicolumn{3}{c}{Pubmed}\\
\cmidrule(r){2-4} \cmidrule(r){5-7} \cmidrule(r){8-10} 
& 25 & 50 & 75 & 25 & 50 & 75 & 25 & 50 & 75  \\

\midrule
GCN & 54.09 $\pm$ 1.48 & 33.55 $\pm$ 2.63 & 22.94 $\pm$ 3.06 &  56.21 $\pm$ 1.26 & 42.41 $\pm$ 2.82 & 32.00 $\pm$ 3.08 & 43.29 $\pm$ 4.75 & 35.91 $\pm$ 3.91 & 35.33 $\pm$ 4.31 \\
GAT & 57.70 $\pm$ 1.30 & 37.76 $\pm$ 3.21 & 20.01 $\pm$ 3.45 & 59.15 $\pm$ 1.03 & 46.74 $\pm$ 3.06 & 37.25 $\pm$ 1.32 & 30.37 $\pm$ 2.15 & 28.08 $\pm$ 3.80 & 26.09 $\pm$ 4.53 \\
APPNP & 56.08 $\pm$ 1.06 & 33.23 $\pm$ 1.48 & 15.61 $\pm$ 0.72 & 60.15 $\pm$ 1.07 & 50.95 $\pm$ 1.91 & 40.15 $\pm$ 1.14 & 38.36 $\pm$ 3.73 & 39.94 $\pm$ 0.00 & 39.94 $\pm$ 0.00  \\
S$^2$GC & 56.02 $\pm$ 0.03 & 31.61 $\pm$ 0.03 & 20.89 $\pm$ 0.57 & 50.18 $\pm$ 0.00 & 34.12 $\pm$ 0.00 & 25.49 $\pm$ 0.03 & 31.59 $\pm$ 2.44 & 39.93 $\pm$ 0.02 & 39.94 $\pm$ 0.00 \\
NAGphormer & 62.11 $\pm$ 1.95 & 41.31 $\pm$ 2.41 & 28.84 $\pm$ 1.36 & 64.19 $\pm$ 1.01 & 56.29 $\pm$ 1.00 & 46.58 $\pm$ 1.15 & 71.49 $\pm$ 1.43 & 45.26 $\pm$ 2.58 & 36.76 $\pm$ 5.64\\
Robust-GCN & 56.23 $\pm$ 0.70 & 36.87 $\pm$ 1.33 & 27.25 $\pm$ 2.23 & 56.67 $\pm$ 0.59 & 41.95 $\pm$ 0.94 & 30.69 $\pm$ 1.45 & 32.43 $\pm$ 1.40 & 33.42 $\pm$ 2.36 & 33.44 $\pm$ 2.58 \\
GCN-Jaccard & 65.70 $\pm$ 1.03 & 46.31 $\pm$ 2.25 & 32.62 $\pm$ 1.31 & 60.28 $\pm$ 1.09 & 47.73 $\pm$ 1.66 & 38.68 $\pm$ 2.66 & 43.79 $\pm$ 4.75 &  35.89 $\pm$ 3.92 & 35.31 $\pm$ 4.31 \\
GCN-SVD & 58.27 $\pm$ 0.97 & 36.49 $\pm$ 1.86 & 25.13 $\pm$ 3.23 & 66.97 $\pm$ 0.84 & 59.21 $\pm$ 0.91 & 40.69 $\pm$ 1.10 & 78.69 $\pm$ 0.50 & 55.05 $\pm$ 0.93 & 36.31 $\pm$ 2.44 \\
Pro-GNN & 72.21 $\pm$ 1.89 & 38.86 $\pm$ 0.78 & 24.34 $\pm$ 3.06 & 67.18 $\pm$ 1.36 & 50.80 $\pm$ 2.05 & 31.60 $\pm$ 2.19  &  - &  -  &  -  \\
GNNGuard & 54.43 $\pm$ 1.45 & 34.48 $\pm$ 2.45 & 25.98 $\pm$ 3.31 & 55.84 $\pm$ 1.68 & 41.23 $\pm$ 1.85 & 31.80 $\pm$ 1.62 & 44.56 $\pm$ 2.91 & 37.56 $\pm$ 4.28 & 37.45 $\pm$ 4.92\\
Elastic GNN & 57.95 $\pm$ 3.28 & 45.18 $\pm$ 1.90 & 30.30 $\pm$ 2.91 & 64.24 $\pm$ 1.03 & 53.19 $\pm$ 2.74 & 42.96 $\pm$ 1.97 & 54.45 $\pm$ 0.60 & 39.94 $\pm$ 0.00 & 39.94 $\pm$ 0.00 \\
STABLE & \textbf{78.69 $\pm$ 0.50} & \textbf{71.94 $\pm$ 0.69} & 62.98 $\pm$ 1.03 & 70.49 $\pm$ 0.95 & 64.04 $\pm$ 2.07 & 52.03 $\pm$ 3.16 & 33.68 $\pm$ 5.38 & 36.57 $\pm$ 5.54 & 35.98 $\pm$ 6.37 \\
EvenNet & 76.30 $\pm$ 0.39 & 71.11 $\pm$ 0.46 & 67.07 $\pm$ 0.60 & 70.89 $\pm$ 0.71 & 66.20 $\pm$ 0.87 & 63.60 $\pm$ 1.71 & 83.98 $\pm$ 0.00 & 81.52 $\pm$ 0.55 & 81.27 $\pm$ 0.46\\
GCN-GARNET & 74.80 $\pm$ 1.19 & 70.90 $\pm$ 1.19 & \textbf{67.49 $\pm$ 1.06} & 70.01 $\pm$ 0.96 & 63.76 $\pm$ 1.99 & 56.64 $\pm$ 2.88 & 85.14 $\pm$ 0.34 & 84.82 $\pm$ 0.47 & 84.74 $\pm$ 0.42 \\
HANG-quad & 67.54 $\pm$ 1.03 & 65.22 $\pm$ 1.07 & 61.88 $\pm$ 1.50 & 66.37 $\pm$ 0.95 & 65.31 $\pm$ 0.93 & 63.68 $\pm$ 1.32 & 85.03 $\pm$ 0.20 & 85.06 $\pm$ 0.20 & 84.89 $\pm$ 0.17\\
\midrule
GADC (IV) & 76.00 $\pm$ 0.61 & 70.29 $\pm$ 0.82 & 66.06 $\pm$ 0.66 & \textbf{71.55 $\pm$ 0.89} & \textbf{66.47 $\pm$ 0.99} & \textbf{65.25 $\pm$ 0.69} & \textbf{86.74 $\pm$ 0.21} & \textbf{85.92 $\pm$ 0.14} & \textbf{85.29 $\pm$ 0.06} \\
\bottomrule
\end{tabular}}
\end{center}
\end{table*}
In this section, we conduct extensive experiments to demonstrate the effectiveness of our proposed GADC architectures in various scenarios.

\subsection{Evaluation of Defense against Adversarial Attacks on the Graph Structure}
\label{sec:defense_attack}
In this section, we demonstrate the defense performance of our proposed GADC (IV) against adversarial attacks on the graph structure, including non-adaptive and adaptive attacks~\citep{mujkanovic2022defenses}, by utilizing the additional term in the modified transition matrix. We conduct experiments on three citation network datasets~\citep{sen2008collective}: Cora, Citeseer, and Pubmed. The statistics for all datasets used in this paper can be found in Appendix~\ref{appendix:dataset}. For non-adaptive attacks, we use Mettack~\citep{zugner2019adversarial}, and for adaptive attacks, we use Aux-Attack.
\begin{table}
 \centering
  \caption{Summary of test accuracy (\%) results from one run under adaptive adversarial attacks.}
  \vspace*{-0.1in}
  \label{tab:adaptive}
\resizebox{\linewidth}{!}{
\setlength{\tabcolsep}{3pt}
\begin{tabular}{ccccc}
        \toprule
        \multirow{2}{*}{Dataset} & \multicolumn{2}{c}{Evasion Test Accuracy} & \multicolumn{2}{c}{Poisoned Test Accuracy}  \\
        \cmidrule(r){2-3} \cmidrule(r){4-5} 
        & Cora &  Citeseer  &  Cora  &  Citeseer    \\
        \midrule
        GCN & 59.71 & 62.78 & 48.99 & 46.98   \\
        SVD-GCN & 57.44 & 58.25 & 45.79 & 49.41   \\
        GNNGuard & 66.35 & 66.50 & 51.07 & 49.70  \\
        Soft-Median-GDC & 67.05 & 63.88 & 56.81 & 58.18 \\
        \midrule
        GADC (IV) & \textbf{72.03} & \textbf{72.94} & \textbf{70.73} & \textbf{68.36}   \\
        \bottomrule
        \end{tabular}}
\end{table}
\paragraph{Baselines.} For the baselines on non-adaptive attacks, we use two popular GNNs: GCN~\citep{kipf2017semi} and GAT~\citep{velivckovic2018graph}; two GDCs: APPNP~\citep{klicpera2019predict} and S$^2$GC~\citep{zhu2021simple}; a graph transformer: NAGphormer~\citep{chen2023nagphormer}; and various defense methods against Meta-attack~\citep{zugner2019adversarial}, including RobustGCN~\citep{zugner2019robustgcn}, GCN-Jaccard~\citep{wu2019adversarial}, GCN-SVD~\citep{entezari2020all}, Pro-GNN~\citep{jin2020graph}, GNNGuard~\citep{zhang2020gnnguard}, Elastic GNN~\citep{liu2021elastic}, STABLE~\citep{li2022reliable}, EvenNet~\citep{lei2022evennet}, GCN-GARNET~\citep{deng2022garnet}, and HANG-quad~\citep{zhao2023adversarial}. For the baselines on adaptive attacks, we consider GCN, GCN-SVD, GNNGuard, and Soft-Median-GDC~\citep{geisler2021robustness}. 

\paragraph{Setting.} 
For non-adaptive attacks, we use DeepRobust~\citep{li2020deeprobust} to generate disrupted graph structures on Cora, Citeseer, and Pubmed datasets with perturbation rates of $0.25,0.5$, and 0.75. We follow dataset splits used in~\citet{jin2020graph}. Each experiment is repeated 10 times, and we report the mean test accuracy and its standard deviation on the node classification task. For our GADC (IV) model, we add a 2-layer MLP with 32 hidden units after feature aggregation. We tune hyper-parameters with the validation dataset. We also tune some baselines such as HANG-quad, GCN-GARNET, and EvenNet. All hyper-parameter details of our methods can be found in Appendix~\ref{appendix:hyperparameter}. For adaptive attacks, we set the attack budget to 1000 for the Cora and Citeseer datasets and use dataset splits from~\citet{mujkanovic2022defenses}. The experiment is repeated once. For our model, we set $K=2$ and discard the terms for $k=0/1$, similar to SGC~\citep{wu2019simplifying}. 
\begin{table*}[t]
\caption{Summary of test accuracy (\%) results from 100 runs on citation network datasets with Gaussian noise.}
\vspace*{-0.1in}
\centering
\resizebox{\linewidth}{!}{%
\setlength{\tabcolsep}{10pt}
\begin{tabular}{@{}c|cccccccc|c@{}}
\hline
\toprule
Cora & MLP & GCN & GAT & APPNP & GLP & S$^2$GC & IRLS & AirGNN & GADC (II)  \\
\midrule
0.1 & 41.0 $\pm$ 9.1 & 53.5 $\pm$ 25.1 & 73.9 $\pm$ 8.7 & 78.1 $\pm$ 8.7 & 65.5 $\pm$ 13.9 & 74.0 $\pm$ 10.5 & 65.8 $\pm$ 12.9 & \textbf{78.8 $\pm$ 6.9} & 77.4 $\pm$ 2.5\\
0.2 & 21.6 $\pm$ 6.5 & 41.3 $\pm$ 20.7 & 62.8 $\pm$ 12.7 & 72.1 $\pm$ 10.3 & 55.4 $\pm$ 12.4 & 65.4 $\pm$ 12.1 & 47.8 $\pm$ 8.9 & \textbf{73.6 $\pm$ 9.4}  & 72.6 $\pm$ 2.8 \\
0.3 & 17.5 $\pm$ 6.4 & 32.8 $\pm$ 13.6 & 51.2 $\pm$ 14.1 & 66.3 $\pm$ 12.4 & 49.9 $\pm$ 11.1 & 60.8 $\pm$ 10.9 & 42.3 $\pm$ 7.7 & 68.5 $\pm$ 11.5 & \textbf{69.1 $\pm$ 3.2} \\
0.4 & 15.9 $\pm$ 6.0 & 30.0 $\pm$ 10.0 & 45.4 $\pm$ 12.3 & 62.3 $\pm$ 12.9 & 49.4 $\pm$ 10.0 & 55.9 $\pm$ 11.8 & 40.9 $\pm$ 7.3 & 65.6 $\pm$ 11.3 & \textbf{68.0 $\pm$ 3.6} \\
0.5 & 14.9 $\pm$ 5.2 & 28.1 $\pm$ 9.1 & 41.9 $\pm$ 13.3 & 63.0 $\pm$ 11.8 & 47.3 $\pm$ 11.7 & 53.4 $\pm$ 11.4 & 41.4 $\pm$ 7.3 & \textbf{67.6 $\pm$ 8.0} & \textbf{67.6 $\pm$ 3.3}  \\
100 & 15.0 $\pm$ 4.6 & 28.4 $\pm$ 7.5 & 31.9 $\pm$ 12.9 & 61.8 $\pm$ 11.2 & 47.6  $\pm$ 10.3 & 52.0 $\pm$ 13.1 & 43.7 $\pm$ 6.6 & 65.4 $\pm$ 11.4 & \textbf{66.9 $\pm$ 5.3}  \\
\midrule
Citeseer & MLP & GCN & GAT & APPNP & GLP & S$^2$GC & IRLS & AirGNN & GADC (II) \\
\midrule
0.1 & 46.3 $\pm$ 3.1 & 52.4 $\pm$ 21.9 & 69.5 $\pm$ 1.1 & 70.3 $\pm$ 1.0 & 65.3 $\pm$ 3.0 & 71.7 $\pm$ 1.1 & 70.8 $\pm$ 3.3 & \textbf{70.9 $\pm$ 1.3} & 69.6 $\pm$ 1.2 \\
0.2 & 25.0 $\pm$ 5.2 & 37.3 $\pm$ 15.9 & 55.1 $\pm$ 10.4 & \textbf{59.6 $\pm$ 9.6} & 47.2 $\pm$ 8.4 & 59.5 $\pm$ 7.3 & 56.0 $\pm$ 7.2 & 58.9 $\pm$ 13.9 & 59.5 $\pm$ 3.5 \\
0.3 & 17.9 $\pm$ 3.2 & 24.4 $\pm$ 4.4 & 36.2 $\pm$ 9.8 & 45.9 $\pm$ 11.7 & 36.4 $\pm$ 7.1 & 46.6 $\pm$ 8.9 & 37.2 $\pm$ 7.0 & 44.2 $\pm$ 14.9 & \textbf{50.5 $\pm$ 3.1} \\
0.4 & 17.4 $\pm$ 3.0 & 23.3 $\pm$ 4.4 & 30.9 $\pm$ 7.1 & 40.8 $\pm$ 10.4 & 36.7 $\pm$ 4.7 & 42.7 $\pm$ 6.7 & 36.3 $\pm$ 4.9 & 39.8 $\pm$ 12.3 & \textbf{48.3 $\pm$ 2.3}  \\
0.5 & 17.4 $\pm$ 2.4 & 22.7 $\pm$ 4.0 & 28.7 $\pm$ 6.2 & 39.5 $\pm$ 8.8 & 36.4 $\pm$ 4.8 & 41.3 $\pm$ 6.8 & 36.7 $\pm$ 5.0 & 36.9 $\pm$ 12.1 & \textbf{47.8 $\pm$ 2.4} \\
100 & 16.8 $\pm$ 2.5 & 23.7 $\pm$ 3.7 & 25.0 $\pm$ 6.5 & 36.6 $\pm$ 9.4 & 38.0 $\pm$ 4.0 & 41.4 $\pm$ 5.3 & 37.3 $\pm$ 4.8 & 33.4 $\pm$ 11.4 & \textbf{46.9 $\pm$ 2.4} \\
\midrule
Pubmed & MLP & GCN & GAT & APPNP & GLP & S$^2$GC & IRLS & AirGNN & GADC (II) \\
\midrule
0.01 & 67.9 $\pm$ 1.4 & 61.1 $\pm$ 18.2 & 77.7 $\pm$ 0.9 & \textbf{80.0 $\pm$ 0.7} & 78.6 $\pm$ 0.9 & 79.2 $\pm$ 0.5 & 81.2 $\pm$ 0.8 & 79.9 $\pm$ 0.4 & 77.4 $\pm$ 0.8 \\
0.02 & 58.3 $\pm$ 2.0 & 61.0 $\pm$ 17.4 & 75.8 $\pm$ 1.3 & 78.3 $\pm$ 1.1 & 76.5 $\pm$ 1.1 & 78.0 $\pm$ 0.8 & 78.4 $\pm$ 1.1 & \textbf{79.3 $\pm$ 0.7} & 76.4 $\pm$ 1.0 \\
0.03 & 47.9 $\pm$ 4.1 & 54.0 $\pm$ 14.2 & 60.7 $\pm$ 13.6 & 74.4 $\pm$ 4.9 & 70.3 $\pm$ 5.7 & 74.4 $\pm$ 2.5 & 69.4 $\pm$ 6.9 & \textbf{76.2 $\pm$ 4.8} & 73.9 $\pm$ 1.6 \\
0.04 & 39.4 $\pm$ 4.8 & 41.7 $\pm$ 9.7 & 40.0 $\pm$ 7.2 & 64.1 $\pm$ 14.4 & 60.8 $\pm$ 8.4 & 62.7 $\pm$ 11.2 & 58.1 $\pm$ 8.7 & 65.9 $\pm$ 13.3 & \textbf{69.0 $\pm$ 2.3}  \\
0.05 & 36.5 $\pm$ 6.6 & 36.9 $\pm$ 6.7 & 38.4 $\pm$ 6.3 & 55.3 $\pm$ 15.9 & 56.9 $\pm$ 7.6 & 55.8 $\pm$ 10.9 & 55.9 $\pm$ 9.7 & 58.9 $\pm$ 14.3 & \textbf{66.3 $\pm$ 2.9}  \\
100 & 35.0 $\pm$ 5.3 & 36.6 $\pm$ 5.3 & 38.2 $\pm$ 3.7 & 52.3 $\pm$ 11.3 & 55.1 $\pm$ 5.6 & 54.0 $\pm$ 10.3 & 50.1 $\pm$ 10.6 & 55.6 $\pm$ 13.6 & \textbf{62.5 $\pm$ 3.6}  \\
\bottomrule
\end{tabular}
}
\label{table:denoising}
\end{table*}

\begin{table}
 \centering
  \caption{Summary of test accuracy (\%) results from 10 runs on the Coauthor-CS and Coauthor-Phy datasets with Gaussian noise.}
  \vspace*{-0.1in}
  \label{tab:coauthor}
\resizebox{\linewidth}{!}{
\setlength{\tabcolsep}{8pt}
\begin{tabular}{ccccc}
        \toprule
        \multirow{2}{*}{Noise Level} & \multicolumn{2}{c}{Coauthor-CS} & \multicolumn{2}{c}{Coauthor-Phy}  \\
        \cmidrule(r){2-3} \cmidrule(r){4-5} 
        & 0.1 &  1  &  0.1  &  1    \\
        \midrule
        MLP & 82.5{$\pm$1.8} & 22.3{$\pm$0.1} & 81.6{$\pm$8.1} & 47.0{$\pm$10.0}  \\
        GCN & 87.3{$\pm$0.5} & 61.3{$\pm$14.3} & 94.2{$\pm$0.4} & 78.6{$\pm$10.6}   \\
        GAT & 86.8{$\pm$3.6} & 57.9{$\pm$20.2} & 94.0{$\pm$0.4} & 63.7{$\pm$16.7}  \\
        APPNP & 94.5{$\pm$0.4} & 81.7{$\pm$2.0} & 95.4{$\pm$0.3} & 89.2{$\pm$1.6}  \\
        GLP & 91.3{$\pm$0.4} & 52.4{$\pm$17.3} & 93.3{$\pm$2.5} & 81.3{$\pm$10.6}  \\
        S$^2$GC & 86.1{$\pm$0.2} & 79.6{$\pm$10.2} & 92.6{$\pm$1.3} & 89.4{$\pm$4.3}  \\
        IRLS & 78.8{$\pm$5.1} & 62.1{$\pm$17.8} & 89.2{$\pm$3.4} & 87.0{$\pm$4.5}  \\
        \midrule
        GADC (II) & \textbf{95.4{$\pm$}0.2} & \textbf{87.8{$\pm$}1.5} & \textbf{95.7{$\pm$}0.2} & \textbf{93.6{$\pm$}0.8}  \\
        \bottomrule
        \end{tabular}}
\end{table}
\paragraph{Results.}
Table~\ref{tab:metaattack} reports the results of non-adaptive attacks for our method and other baselines. Our method outperforms other baselines on Citeseer and Pubmed and achieves performance relatively close to the best baseline on Cora. Additionally, our GADC (IV) significantly outperforms both APPNP and S$^2$GC. As the perturbation rate increases, the adjacency matrix becomes increasingly unreliable. However, our additional term leverages the inherent homophily properties of these graphs to reconstruct the clean adjacency matrix, thereby restoring effective information flow within the aggregation scheme. GNNGuard and NAGphormer also evaluate the importance of neighbors using the attention mechanism on node embeddings at each layer to mitigate the disruption of the graph structure caused by adversarial attacks. In contrast, our algorithm is a pre-computation method that uses the additional term to reconstruct the graph structure directly in one step. After reconstructing the adjacency matrix, we use it directly to aggregate features. We think that smoothing node features makes the recalculated weights less accurate. Therefore, our method performs much better than GNNGuard and NAGphormer.

Table~\ref{tab:adaptive} reports the experimental results of our method and the baselines under adaptive attacks. The results demonstrate that our method significantly outperforms the baselines. This indicates that even when adaptive attacks modify the graph structure during inference, our additional terms can still effectively reconstruct the adjacency matrix and provide a robust defense.
\subsection{Evaluation of Denoising against Noise in Node Features}
\label{sec:denoising}
In this section, we compare the denoising performance of our proposed GADC (II, III) with various baselines by evaluating their test accuracy when models are trained on noisy feature matrices.
\begin{table}
\centering
\caption{Summary of test accuracy (\%) results from 10 runs on the ogbn-products dataset with Gaussian noise.}
\vspace*{-0.1in}
\label{tab:ogb}
\resizebox{\linewidth}{!}{\setlength{\tabcolsep}{20pt}
\begin{tabular}{ccc}
    \toprule
    \multirow{2}{*}{Noise Level} & \multicolumn{2}{c}{ogbn-products}  \\
    \cmidrule(r){2-3}
     & 0.1 & 1   \\
    \midrule
    MLP & 59.68{$\pm$0.16}  & 38.08{$\pm$0.10} \\
    GCN & 75.60{$\pm$0.19} & 72.76{$\pm$0.20}  \\
    S$^2$GC &  74.95{$\pm$0.13}  & 63.17{$\pm$0.12}  \\
    \midrule
    GADC (III) & \textbf{77.54{$\pm$}0.15}  & \textbf{73.66{$\pm$}0.13}  \\
    \bottomrule
    \end{tabular}}
\end{table}
\paragraph{Datasets.} For our experiments, we use three small-scale graph datasets: Cora, Citeseer, and Pubmed, and three large-scale graph datasets: Coauthor-CS, Coauthor-Phy~\citep{shchur2018pitfalls}, and ogbn-products~\citep{hu2020open}. We follow dataset splits used in~\citep{yang2016revisiting} for the citation network datasets. For the Coauthor datasets, we split the nodes into 60\% for training, 20\% for validation, and 20\% for testing. For the ogbn-products dataset, we follow dataset splits provided by OGB~\citep{hu2020open}.
\begin{table*}[t]
\normalsize
\begin{center}
\caption{Summary of ablation study results in terms of test accuracy (\%).}
 \vspace*{-0.1in}
\label{ablation denoising}
\resizebox{\linewidth}{!}{
\setlength{\tabcolsep}{8pt}
\begin{tabular}{ccccccccc}
\toprule
\multicolumn{1}{c}{Dataset}& \multicolumn{2}{c}{Cora} & \multicolumn{2}{c}{Coauthor-CS} & \multicolumn{2}{c}{Coauthor-Phy} & \multicolumn{2}{c}{ogbn-products}  \\
\cmidrule(r){1-1} \cmidrule(r){2-3} \cmidrule(r){4-5} \cmidrule(r){6-7} \cmidrule(r){8-9} 
Noise Level&  0.1     &  0.5  &   0.1 & 1.0
&  0.1     &  1.0  &   0.1 & 1.0\\
\midrule
GADC (II/III, $\varepsilon=0$)  &76.4 $\pm$ 3.2 & 66.5 $\pm$ 5.1 & 95.3 $\pm$ 0.2   & 87.1 $\pm$ 3.1  & 95.7 $\pm$ 0.2 & 93.1 $\pm$ 1.4  & 77.5 $\pm$ 0.2  & 73.4 $\pm$ 0.1    \\
GADC (II/III, $\varepsilon\neq 0$)  & 77.4 $\pm$ 2.5 & 67.6 $\pm$ 3.3 & 95.4 $\pm$ 0.2   &87.8 $\pm$ 1.5  & 95.7 $\pm$ 0.2  & 93.6 $\pm$ 0.8  & 77.5 $\pm$ 0.2  &73.7 $\pm$ 0.1   \\
\bottomrule
\end{tabular}
}
\end{center}
\end{table*}
\paragraph{Baselines.}
For the baselines, we consider several variants of GDC, including APPNP~\citep{klicpera2019predict}, GLP~\citep{li2019label}, S$^2$GC~\citep{zhu2021simple}, and AirGNN~\citep{liu2021graph}. Additionally, we include popular GNNs such as GCN~\citep{kipf2017semi} and GAT~\citep{velivckovic2018graph}. We also consider IRLS~\citep{yang2021graph}, a GNN architecture derived from GSD, and an MLP, which does not employ any aggregation scheme. 

\paragraph{Setting.}
We assume that the original feature matrix is clean and devoid of noise. To introduce noise, we synthesize it from a standard Gaussian distribution and add it to the original feature matrix, as suggested in AirGNN~\citep{liu2021graph}. After adding Gaussian noise, we apply row normalization to node features and train all models using these noisy feature matrices. The noise level $\xi$ controls the magnitude of the noise added to the feature matrix: $\mathbf{X}+\xi\bm{\Upsilon}$, where $\bm{\Upsilon}$ is sampled from a standard i.i.d. Gaussian distribution. For Cora and Citeseer, we test $\xi \in \{0.1, 0.2, 0.3, 0.4, 0.5, 100\}$, and for Pubmed, we test $\xi \in \{0.01, 0.02, 0.03, 0.04, 0.05, 100\}$. For Coauthor-CS, Coauthor-Phy, and ogbn-products, we test $\xi \in \{0.1, 1.0\}$. For each model's hyperparameters, we follow the settings reported in their original papers. To reduce randomness, we repeat the experiment 100 or 10 times and report the mean test accuracy. In each repeated run, different Gaussian noises are added; however, within the same run, the same noisy feature matrix is used to train all models. For the citation network datasets, we employ the second option in Eq.~\eqref{eq:options}, and we set $K=16$ and $\lambda=32$ by default. We select $K$ for other datasets based on the best performance on the validation dataset. For coauthor datasets, we use the second option. For ogbn-products, we use the third option and don't add additional graph connectivity since it is a very large graph. Therefore, we only compute $\Phi_{i j}$ based on connected neighbors in the original graph.

\paragraph{Results.}
Table~\ref{table:denoising}, \ref{tab:coauthor}, and \ref{tab:ogb} report test accuracy results across various noise levels for node classification tasks. As demonstrated in Table~\ref{table:denoising}, \ref{tab:coauthor}, and \ref{tab:ogb}, under high noise levels, the performance of MLP approximates random guessing, as the test accuracy is close to the inverse of the total number of labels. This implies that the added noise has effectively obscured the original features. Compared to GDC variants, our proposed GADC (II, III) demonstrates superior denoising performance at high noise levels. This demonstrates the effectiveness of the additional term in our modified transition matrix for tracking large noise. We also synthesize noise by flipping individual features with a small Bernoulli probability on three citation network datasets. We report the results in Appendix~\ref{appendix:result}.

\paragraph{Computational Complexity.} Introducing additional edges in GADC (II) does increase the time complexity of the aggregation process. However, we mitigate this issue by decoupling the aggregation process from downstream training, ensuring that aggregation occurs only once. This differs from APPNP, which repeats the aggregation process during each training iteration (i.e., aggregation occurs $n$ times for $n$ iterations). As a result, completing 100 runs (one experiment) of GADC (II) on the Pubmed dataset takes only 58 seconds, and completing 10 runs on the Coauthor-CS dataset takes 22 seconds. Without introducing additional edges ($\varepsilon=0$), 100 runs on the Pubmed dataset take only 41 seconds. In contrast, APPNP requires 232 seconds to complete 100 runs on the Pubmed dataset. This demonstrates the computational efficiency of our method. For the ogbn-products dataset, we maintain only the one-hop connections and adjust the edge weights, without introducing additional edges, resulting in a time complexity similar to that of GDC.

\subsection{Ablation Study on the Effectiveness of the Additional Term}
We conduct an ablation study by setting $\varepsilon$ to zero to observe its impact. As shown in Table~\ref{ablation denoising}, for the low node-degree graph (Cora), our additional term effectively enhances denoising performance under both small and large noise. However, for high node-degree graphs (Coauthor and ogbn-products), performance improvement is noticeable only under large noise. This is because low node-degree graphs have relatively fewer neighboring nodes, so introducing the additional term increases graph connectivity, effectively filtering out noise. In high node-degree graphs, the inherent connectivity is already strong enough under small noise, reducing the need for additional connectivity to filter out noise. However, under large noise, enhanced connectivity still improves denoising for Co-author. In the ogbn-products graph, although GADC (III, $\varepsilon \neq 0)$ doesn't introduce new connections, it reassigns the weights of existing edges, successfully reducing the value of $\tau$ and thereby improving denoising performance. 

\subsection{Improving Performance on Heterophilic Graphs}
\label{heterophilic graph}
We conduct an ablation study to show that our GADC (I) can improve performance on heterophilic graphs. Due to the space limit, we provide the experimental results in Appendix~\ref{appendix:heterophilic} and discuss related work in Appendix~\ref{appendix:related}.

\section{Conclusion}
We introduce a min-max formulation of the graph signal denoising problem and develop various GADC variants. Extensive results demonstrate that our GADC effectively addresses noisy features in graphs, graph structure attacks, and inconsistent edges on heterophilic graphs.

\section*{Acknowledgements}
We thank all the anonymous reviewers for their helpful comments and suggestions. Songtao Liu thanks ICML and NeurIPS so much for providing financial aid, making him attend ICML 2022/2023 and NeurIPS 2023 in person, as he was unable to get other funding sources for attendance (ICML 2023, NeurIPS 2023).

\section*{Impact Statement}
This paper presents work whose goal is to advance the field of 
Machine Learning. There are many potential societal consequences of our work, none which we feel must be specifically highlighted here.


\bibliography{example_paper}

\begin{thebibliography}{61}
\providecommand{\natexlab}[1]{#1}
\providecommand{\url}[1]{\texttt{#1}}
\expandafter\ifx\csname urlstyle\endcsname\relax
  \providecommand{\doi}[1]{doi: #1}\else
  \providecommand{\doi}{doi: \begingroup \urlstyle{rm}\Url}\fi

\bibitem[Brin(1998)]{brin1998pagerank}
Brin, S.
\newblock The pagerank citation ranking: bringing order to the web.
\newblock \emph{Proceedings of ASIS, 1998}, 98:\penalty0 161--172, 1998.

\bibitem[Chang et~al.(2021)Chang, Rong, Xu, Bian, Zhou, Wang, Huang, and Zhu]{chang2021not}
Chang, H., Rong, Y., Xu, T., Bian, Y., Zhou, S., Wang, X., Huang, J., and Zhu, W.
\newblock Not all low-pass filters are robust in graph convolutional networks.
\newblock In \emph{Advances in Neural Information Processing Systems}, 2021.

\bibitem[Chang et~al.(2022)Chang, Rong, Xu, Huang, Zhang, Cui, Wang, Zhu, and Huang]{chang2022adversarial}
Chang, H., Rong, Y., Xu, T., Huang, W., Zhang, H., Cui, P., Wang, X., Zhu, W., and Huang, J.
\newblock Adversarial attack framework on graph embedding models with limited knowledge.
\newblock \emph{IEEE Transactions on Knowledge and Data Engineering}, 35\penalty0 (5):\penalty0 4499--4513, 2022.

\bibitem[Chen et~al.(2023)Chen, Gao, Li, and He]{chen2023nagphormer}
Chen, J., Gao, K., Li, G., and He, K.
\newblock {NAG}phormer: A tokenized graph transformer for node classification in large graphs.
\newblock In \emph{International Conference on Learning Representations}, 2023.

\bibitem[Chen et~al.(2022)Chen, Wang, Wang, Yang, and Lin]{chen2022optimization}
Chen, Q., Wang, Y., Wang, Y., Yang, J., and Lin, Z.
\newblock Optimization-induced graph implicit nonlinear diffusion.
\newblock In \emph{International Conference on Machine Learning}, 2022.

\bibitem[Chen et~al.(2014)Chen, Sandryhaila, Moura, and Kovacevic]{chen2014signal}
Chen, S., Sandryhaila, A., Moura, J.~M., and Kovacevic, J.
\newblock Signal denoising on graphs via graph filtering.
\newblock In \emph{2014 IEEE Global Conference on Signal and Information Processing (GlobalSIP)}, pp.\  872--876. IEEE, 2014.

\bibitem[Chen et~al.(2021)Chen, Eldar, and Zhao]{chen2021graph}
Chen, S., Eldar, Y.~C., and Zhao, L.
\newblock Graph unrolling networks: Interpretable neural networks for graph signal denoising.
\newblock \emph{IEEE Transactions on Signal Processing}, 69:\penalty0 3699--3713, 2021.

\bibitem[Dai et~al.(2018)Dai, Li, Tian, Huang, Wang, Zhu, and Song]{dai2018adversarial}
Dai, H., Li, H., Tian, T., Huang, X., Wang, L., Zhu, J., and Song, L.
\newblock Adversarial attack on graph structured data.
\newblock In \emph{International Conference on Machine Learning}, 2018.

\bibitem[Dai et~al.(2019)Dai, Li, Coley, Dai, and Song]{dai2019retrosynthesis}
Dai, H., Li, C., Coley, C.~W., Dai, B., and Song, L.
\newblock Retrosynthesis prediction with conditional graph logic network.
\newblock In \emph{Advances in Neural Information Processing Systems}, 2019.

\bibitem[Deng et~al.(2022)Deng, Li, Feng, and Zhang]{deng2022garnet}
Deng, C., Li, X., Feng, Z., and Zhang, Z.
\newblock Garnet: Reduced-rank topology learning for robust and scalable graph neural networks.
\newblock In \emph{Learning on Graphs Conference}, 2022.

\bibitem[Entezari et~al.(2020)Entezari, Al-Sayouri, Darvishzadeh, and Papalexakis]{entezari2020all}
Entezari, N., Al-Sayouri, S.~A., Darvishzadeh, A., and Papalexakis, E.~E.
\newblock All you need is low (rank) defending against adversarial attacks on graphs.
\newblock In \emph{Proceedings of the 13th International Conference on Web Search and Data Mining}, 2020.

\bibitem[Fey \& Lenssen(2019)Fey and Lenssen]{fey2019fast}
Fey, M. and Lenssen, J.~E.
\newblock Fast graph representation learning with pytorch geometric.
\newblock \emph{arXiv preprint arXiv:1903.02428}, 2019.

\bibitem[Gasteiger et~al.(2019)Gasteiger, Wei{\ss}enberger, and G{\"u}nnemann]{gasteiger2019diffusion}
Gasteiger, J., Wei{\ss}enberger, S., and G{\"u}nnemann, S.
\newblock Diffusion improves graph learning.
\newblock In \emph{Advances in Neural Information Processing Systems}, 2019.

\bibitem[Geisler et~al.(2021)Geisler, Schmidt, {\c{S}}irin, Z{\"u}gner, Bojchevski, and G{\"u}nnemann]{geisler2021robustness}
Geisler, S., Schmidt, T., {\c{S}}irin, H., Z{\"u}gner, D., Bojchevski, A., and G{\"u}nnemann, S.
\newblock Robustness of graph neural networks at scale.
\newblock In \emph{Advances in Neural Information Processing Systems}, 2021.

\bibitem[Gosch et~al.(2023)Gosch, Geisler, Sturm, Charpentier, Z{\"u}gner, and G{\"u}nnemann]{gosch2023adversarial}
Gosch, L., Geisler, S., Sturm, D., Charpentier, B., Z{\"u}gner, D., and G{\"u}nnemann, S.
\newblock Adversarial training for graph neural networks: Pitfalls, solutions, and new directions.
\newblock In \emph{Advances in Neural Information Processing Systems}, 2023.

\bibitem[Guo et~al.(2019)Guo, Lin, Feng, Song, and Wan]{guo2019attention}
Guo, S., Lin, Y., Feng, N., Song, C., and Wan, H.
\newblock Attention based spatial-temporal graph convolutional networks for traffic flow forecasting.
\newblock In \emph{Proceedings of the AAAI Conference on Artificial Intelligence}, 2019.

\bibitem[Hamilton et~al.(2017)Hamilton, Ying, and Leskovec]{hamilton2017inductive}
Hamilton, W., Ying, Z., and Leskovec, J.
\newblock Inductive representation learning on large graphs.
\newblock In \emph{Advances in Neural Information Processing Systems}, 2017.

\bibitem[Hoeffding(1994)]{hoeffding1994probability}
Hoeffding, W.
\newblock Probability inequalities for sums of bounded random variables.
\newblock \emph{The Collected Works of Wassily Hoeffding}, pp.\  409--426, 1994.

\bibitem[Hu et~al.(2020)Hu, Fey, Zitnik, Dong, Ren, Liu, Catasta, and Leskovec]{hu2020open}
Hu, W., Fey, M., Zitnik, M., Dong, Y., Ren, H., Liu, B., Catasta, M., and Leskovec, J.
\newblock Open graph benchmark: Datasets for machine learning on graphs.
\newblock \emph{arXiv preprint arXiv:2005.00687}, 2020.

\bibitem[Jia \& Benson(2022)Jia and Benson]{jia2022unifying}
Jia, J. and Benson, A.~R.
\newblock A unifying generative model for graph learning algorithms: Label propagation, graph convolutions, and combinations.
\newblock \emph{SIAM Journal on Mathematics of Data Science}, 4\penalty0 (1):\penalty0 100--125, 2022.

\bibitem[Jin et~al.(2020)Jin, Ma, Liu, Tang, Wang, and Tang]{jin2020graph}
Jin, W., Ma, Y., Liu, X., Tang, X., Wang, S., and Tang, J.
\newblock Graph structure learning for robust graph neural networks.
\newblock In \emph{Proceedings of the 26th ACM SIGKDD International Conference on Knowledge Discovery \& Data Mining}, 2020.

\bibitem[Kipf \& Welling(2017)Kipf and Welling]{kipf2017semi}
Kipf, T.~N. and Welling, M.
\newblock Semi-supervised classification with graph convolutional networks.
\newblock In \emph{International Conference on Learning Representations}, 2017.

\bibitem[Klicpera et~al.(2019)Klicpera, Bojchevski, and G{\"u}nnemann]{klicpera2019predict}
Klicpera, J., Bojchevski, A., and G{\"u}nnemann, S.
\newblock Predict then propagate: Graph neural networks meet personalized pagerank.
\newblock In \emph{International Conference on Learning Representations}, 2019.

\bibitem[Lei et~al.(2022)Lei, Wang, Li, Ding, and Wei]{lei2022evennet}
Lei, R., Wang, Z., Li, Y., Ding, B., and Wei, Z.
\newblock Evennet: Ignoring odd-hop neighbors improves robustness of graph neural networks.
\newblock In \emph{Advances in Neural Information Processing Systems}, 2022.

\bibitem[Li et~al.(2022)Li, Liu, Ao, Chi, Feng, Yang, and He]{li2022reliable}
Li, K., Liu, Y., Ao, X., Chi, J., Feng, J., Yang, H., and He, Q.
\newblock Reliable representations make a stronger defender: Unsupervised structure refinement for robust gnn.
\newblock In \emph{Proceedings of the 28th ACM SIGKDD Conference on Knowledge Discovery and Data Mining}, 2022.

\bibitem[Li et~al.(2019)Li, Wu, Liu, Zhang, and Guan]{li2019label}
Li, Q., Wu, X.-M., Liu, H., Zhang, X., and Guan, Z.
\newblock Label efficient semi-supervised learning via graph filtering.
\newblock In \emph{Proceedings of the IEEE/CVF Conference on Computer Vision and Pattern Recognition}, 2019.

\bibitem[Li et~al.(2020)Li, Jin, Xu, and Tang]{li2020deeprobust}
Li, Y., Jin, W., Xu, H., and Tang, J.
\newblock Deeprobust: A pytorch library for adversarial attacks and defenses.
\newblock \emph{arXiv preprint arXiv:2005.06149}, 2020.

\bibitem[Lin et~al.(2022)Lin, Blaser, and Wang]{lin2022graph}
Lin, L., Blaser, E., and Wang, H.
\newblock Graph structural attack by perturbing spectral distance.
\newblock In \emph{Proceedings of the 28th ACM SIGKDD Conference on Knowledge Discovery and Data Mining}, 2022.

\bibitem[Liu et~al.(2022)Liu, Ying, Dong, Lin, Chen, and Wu]{liu2022powerful}
Liu, S., Ying, R., Dong, H., Lin, L., Chen, J., and Wu, D.
\newblock How powerful is implicit denoising in graph neural networks.
\newblock \emph{arXiv preprint arXiv:2209.14514}, 2022.

\bibitem[Liu et~al.(2021{\natexlab{a}})Liu, Ding, Jin, Xu, Ma, Liu, and Tang]{liu2021graph}
Liu, X., Ding, J., Jin, W., Xu, H., Ma, Y., Liu, Z., and Tang, J.
\newblock Graph neural networks with adaptive residual.
\newblock In \emph{Advances in Neural Information Processing Systems}, 2021{\natexlab{a}}.

\bibitem[Liu et~al.(2021{\natexlab{b}})Liu, Jin, Ma, Li, Liu, Wang, Yan, and Tang]{liu2021elastic}
Liu, X., Jin, W., Ma, Y., Li, Y., Liu, H., Wang, Y., Yan, M., and Tang, J.
\newblock Elastic graph neural networks.
\newblock In \emph{International Conference on Machine Learning}, 2021{\natexlab{b}}.

\bibitem[Ma et~al.(2021)Ma, Liu, Zhao, Liu, Tang, and Shah]{ma2021unified}
Ma, Y., Liu, X., Zhao, T., Liu, Y., Tang, J., and Shah, N.
\newblock A unified view on graph neural networks as graph signal denoising.
\newblock In \emph{Proceedings of the 30th ACM International Conference on Information \& Knowledge Management}, 2021.

\bibitem[Madry et~al.(2018)Madry, Makelov, Schmidt, Tsipras, and Vladu]{madry2018towards}
Madry, A., Makelov, A., Schmidt, L., Tsipras, D., and Vladu, A.
\newblock Towards deep learning models resistant to adversarial attacks.
\newblock In \emph{International Conference on Learning Representations}, 2018.

\bibitem[Mor{\'e} \& Sorensen(1982)Mor{\'e} and Sorensen]{more1982newton}
Mor{\'e}, J.~J. and Sorensen, D.~C.
\newblock Newton's method.
\newblock Technical report, Argonne National Lab., IL (USA), 1982.

\bibitem[Mujkanovic et~al.(2022)Mujkanovic, Geisler, G{\"u}nnemann, and Bojchevski]{mujkanovic2022defenses}
Mujkanovic, F., Geisler, S., G{\"u}nnemann, S., and Bojchevski, A.
\newblock Are defenses for graph neural networks robust?
\newblock In \emph{Advances in Neural Information Processing Systems}, 2022.

\bibitem[Nt \& Maehara(2019)Nt and Maehara]{nt2019revisiting}
Nt, H. and Maehara, T.
\newblock Revisiting graph neural networks: All we have is low-pass filters.
\newblock \emph{arXiv preprint arXiv:1905.09550}, 2019.

\bibitem[Paszke et~al.(2019)Paszke, Gross, Massa, Lerer, Bradbury, Chanan, Killeen, Lin, Gimelshein, Antiga, et~al.]{paszke2019pytorch}
Paszke, A., Gross, S., Massa, F., Lerer, A., Bradbury, J., Chanan, G., Killeen, T., Lin, Z., Gimelshein, N., Antiga, L., et~al.
\newblock Pytorch: An imperative style, high-performance deep learning library.
\newblock In \emph{Advances in Neural Information Processing Systems}, 2019.

\bibitem[Pei et~al.(2020)Pei, Wei, Chang, Lei, and Yang]{pei2020geom}
Pei, H., Wei, B., Chang, K. C.-C., Lei, Y., and Yang, B.
\newblock Geom-gcn: Geometric graph convolutional networks.
\newblock In \emph{International Conference on Learning Representations}, 2020.

\bibitem[Sen et~al.(2008)Sen, Namata, Bilgic, Getoor, Galligher, and Eliassi-Rad]{sen2008collective}
Sen, P., Namata, G., Bilgic, M., Getoor, L., Galligher, B., and Eliassi-Rad, T.
\newblock Collective classification in network data.
\newblock \emph{AI Magazine}, 2008.

\bibitem[Shchur et~al.(2018)Shchur, Mumme, Bojchevski, and G{\"u}nnemann]{shchur2018pitfalls}
Shchur, O., Mumme, M., Bojchevski, A., and G{\"u}nnemann, S.
\newblock Pitfalls of graph neural network evaluation.
\newblock \emph{arXiv preprint arXiv:1811.05868}, 2018.

\bibitem[Veli{\v{c}}kovi{\'c} et~al.(2018)Veli{\v{c}}kovi{\'c}, Cucurull, Casanova, Romero, Lio, and Bengio]{velivckovic2018graph}
Veli{\v{c}}kovi{\'c}, P., Cucurull, G., Casanova, A., Romero, A., Lio, P., and Bengio, Y.
\newblock Graph attention networks.
\newblock In \emph{International Conference on Learning Representations}, 2018.

\bibitem[Vershynin(2010)]{vershynin2010introduction}
Vershynin, R.
\newblock Introduction to the non-asymptotic analysis of random matrices.
\newblock \emph{arXiv preprint arXiv:1011.3027}, 2010.

\bibitem[Wang et~al.(2022)Wang, Wang, Zhu, and Wang]{wang2022improving}
Wang, Q., Wang, Y., Zhu, H., and Wang, Y.
\newblock Improving out-of-distribution generalization by adversarial training with structured priors.
\newblock In \emph{Advances in Neural Information Processing Systems}, 2022.

\bibitem[Wang et~al.(2021)Wang, Wang, Yang, and Lin]{wang2021dissecting}
Wang, Y., Wang, Y., Yang, J., and Lin, Z.
\newblock Dissecting the diffusion process in linear graph convolutional networks.
\newblock In \emph{Advances in Neural Information Processing Systems}, 2021.

\bibitem[Wu et~al.(2019{\natexlab{a}})Wu, Souza, Zhang, Fifty, Yu, and Weinberger]{wu2019simplifying}
Wu, F., Souza, A., Zhang, T., Fifty, C., Yu, T., and Weinberger, K.
\newblock Simplifying graph convolutional networks.
\newblock In \emph{International Conference on Machine Learning}, 2019{\natexlab{a}}.

\bibitem[Wu et~al.(2019{\natexlab{b}})Wu, Wang, Tyshetskiy, Docherty, Lu, and Zhu]{wu2019adversarial}
Wu, H., Wang, C., Tyshetskiy, Y., Docherty, A., Lu, K., and Zhu, L.
\newblock Adversarial examples for graph data: Deep insights into attack and defense.
\newblock In \emph{Proceedings of the Twenty-Eighth International Joint Conference on Artificial Intelligence}, 2019{\natexlab{b}}.

\bibitem[Xie et~al.(2023)Xie, Chang, Zhang, Wang, Wang, Zhang, Ying, and Zhu]{xie2023adversarially}
Xie, B., Chang, H., Zhang, Z., Wang, X., Wang, D., Zhang, Z., Ying, R., and Zhu, W.
\newblock Adversarially robust neural architecture search for graph neural networks.
\newblock In \emph{Proceedings of the IEEE/CVF Conference on Computer Vision and Pattern Recognition}, 2023.

\bibitem[Yang et~al.(2021)Yang, Liu, Wang, Zhou, Gan, Wei, Zhang, Huang, and Wipf]{yang2021graph}
Yang, Y., Liu, T., Wang, Y., Zhou, J., Gan, Q., Wei, Z., Zhang, Z., Huang, Z., and Wipf, D.
\newblock Graph neural networks inspired by classical iterative algorithms.
\newblock In \emph{International Conference on Machine Learning}, 2021.

\bibitem[Yang et~al.(2016)Yang, Cohen, and Salakhudinov]{yang2016revisiting}
Yang, Z., Cohen, W., and Salakhudinov, R.
\newblock Revisiting semi-supervised learning with graph embeddings.
\newblock In \emph{International Conference on Machine Learning}, 2016.

\bibitem[Ying et~al.(2018)Ying, He, Chen, Eksombatchai, Hamilton, and Leskovec]{ying2018graph}
Ying, R., He, R., Chen, K., Eksombatchai, P., Hamilton, W.~L., and Leskovec, J.
\newblock Graph convolutional neural networks for web-scale recommender systems.
\newblock In \emph{Proceedings of the 24th ACM SIGKDD International Conference on Knowledge Discovery \& Data Mining}, 2018.

\bibitem[Zhang et~al.(2022)Zhang, Liu, Sun, and Shah]{zhang2022graphless}
Zhang, S., Liu, Y., Sun, Y., and Shah, N.
\newblock Graph-less neural networks: Teaching old {MLP}s new tricks via distillation.
\newblock In \emph{International Conference on Learning Representations}, 2022.

\bibitem[Zhang \& Zitnik(2020)Zhang and Zitnik]{zhang2020gnnguard}
Zhang, X. and Zitnik, M.
\newblock Gnnguard: Defending graph neural networks against adversarial attacks.
\newblock In \emph{Advances in Neural Information Processing Systems}, 2020.

\bibitem[Zhao et~al.(2021)Zhao, Dong, Ding, Kharlamov, and Tang]{zhao2021adaptive}
Zhao, J., Dong, Y., Ding, M., Kharlamov, E., and Tang, J.
\newblock Adaptive diffusion in graph neural networks.
\newblock In \emph{Advances in Neural Information Processing Systems}, 2021.

\bibitem[Zhao et~al.(2023)Zhao, Kang, Song, She, Wang, and Tay]{zhao2023adversarial}
Zhao, K., Kang, Q., Song, Y., She, R., Wang, S., and Tay, W.~P.
\newblock Adversarial robustness in graph neural networks: A hamiltonian approach.
\newblock In \emph{Advances in Neural Information Processing Systems}, 2023.

\bibitem[Zhao \& Akoglu(2019)Zhao and Akoglu]{zhao2019pairnorm}
Zhao, L. and Akoglu, L.
\newblock Pairnorm: Tackling oversmoothing in gnns.
\newblock In \emph{International Conference on Learning Representations}, 2019.

\bibitem[Zhou et~al.(2021)Zhou, Li, Zheng, Wang, and Gao]{zhou2021graph}
Zhou, B., Li, R., Zheng, X., Wang, Y.~G., and Gao, J.
\newblock Graph denoising with framelet regularizer.
\newblock \emph{arXiv preprint arXiv:2111.03264}, 2021.

\bibitem[Zhu et~al.(2019)Zhu, Zhang, Cui, and Zhu]{zhu2019robust}
Zhu, D., Zhang, Z., Cui, P., and Zhu, W.
\newblock Robust graph convolutional networks against adversarial attacks.
\newblock In \emph{Proceedings of the 25th ACM SIGKDD International Conference on Knowledge Discovery \& Data Mining}, 2019.

\bibitem[Zhu \& Koniusz(2021)Zhu and Koniusz]{zhu2021simple}
Zhu, H. and Koniusz, P.
\newblock Simple spectral graph convolution.
\newblock In \emph{International Conference on Learning Representations}, 2021.

\bibitem[Z{\"u}gner \& G{\"u}nnemann(2019{\natexlab{a}})Z{\"u}gner and G{\"u}nnemann]{zugner2019adversarial}
Z{\"u}gner, D. and G{\"u}nnemann, S.
\newblock Adversarial attacks on graph neural networks via meta learning.
\newblock In \emph{International Conference on Learning Representations}, 2019{\natexlab{a}}.

\bibitem[Z{\"u}gner \& G{\"u}nnemann(2019{\natexlab{b}})Z{\"u}gner and G{\"u}nnemann]{zugner2019robustgcn}
Z{\"u}gner, D. and G{\"u}nnemann, S.
\newblock Certifiable robustness and robust training for graph convolutional networks.
\newblock In \emph{Proceedings of the 25th ACM SIGKDD International Conference on Knowledge Discovery \& Data Mining}, 2019{\natexlab{b}}.

\bibitem[Z{\"u}gner et~al.(2018)Z{\"u}gner, Akbarnejad, and G{\"u}nnemann]{zugner2018adversarial}
Z{\"u}gner, D., Akbarnejad, A., and G{\"u}nnemann, S.
\newblock Adversarial attacks on neural networks for graph data.
\newblock In \emph{Proceedings of the 24th ACM SIGKDD International Conference on Knowledge Discovery \& Data Mining}, 2018.

\end{thebibliography}
\bibliographystyle{icml2024}

\newpage
\appendix
\onecolumn
\section{The Row Summation of the Graph Diffusion Convolution Matrix}
\label{appendix:row sum}
We provide the derivations of the row sum of $\mathbf{S}=\frac{1}{\lambda+1} \sum_{k=0}^K\left[\frac{\lambda}{\lambda+1} \widetilde{\bm{\mathcal{A}}}\right]^k$ in this section. Before we derive the row summation of $\mathbf{S}$, we first derive the row summation of $\widetilde{\bm{\mathcal{A}}}^{k}$. Note we consider $\widetilde{\bm{\mathcal{A}}}=\widetilde{\mathbf{D}}^{-1} \widetilde{\mathbf{A}}$ for the simplicity of the proof.
\begin{lemma}
\label{transition matrix}
Consider a probability matrix $\mathbf{P} \in \mathbb{R}^{n \times n}$, where $\mathbf{P}_{ij}\geq0$. Besides, for all $i$, we have $\sum_{j=1}^n\mathbf{P}_{ij}=1$. Then for any $k \in \mathbb{Z}_{+}$, we have $\sum_{j=1}^n\mathbf{P}^k_{ij}=1$,
\end{lemma}
\begin{proof}
We give a proof by induction on $k$.\\
\textbf{Base case:} When $k=1$, the case is true.\\
\textbf{Inductive step:} Assume the induction hypothesis that for a particular $k$, the single case n = k holds, meaning $\mathbf{P}^k$ is true:
\begin{equation*}
    \forall i, \sum_{j=1}^n \mathbf{P}_{ij}^k =1.
\end{equation*}
As $\mathbf{P}^{k+1}=\mathbf{P}^{k}\mathbf{P}$, so we have
\begin{equation*}
\begin{split}
\sum_{j=1}^n\mathbf{P}^{k+1}_{ij}&=\sum_{j=1}^n\sum_{k=1}^n\mathbf{P}^{k}_{ik}\mathbf{P}_{kj} \\
&=  \sum_{k=1}^n\sum_{j=1}^n\mathbf{P}^{k}_{ik}\mathbf{P}_{kj}\\
&= \sum_{k=1}^n\mathbf{P}^{k}_{ik}\left(\sum_{j=1}^n\mathbf{P}_{kj}\right) \\&
= \sum_{k=1}^n\mathbf{P}^{k}_{ik} = 1,
\end{split} 
\end{equation*}
which finishes the proof.
\end{proof}

Lemma~\ref{transition matrix} describes the row summation of $\widetilde{\bm{\mathcal{A}}}^{k}$ is 1. Now we can obtain the row summation for $\mathbf{S}$.

Then for any $i$, we have
\begin{equation}
\begin{split}
    \sum_{j=1}^{n}\left[\mathbf{S}\right]_{ij} 
    &=\frac{1}{\lambda+1}\sum_{k=0}^{K} \left(\frac{\lambda}{\lambda+1}\left[\widetilde{\bm{\mathcal{A}}}\right]_{ij}\right)^{k}\\
    &=\frac{1}{\lambda+1}\sum_{k=0}^{K}\left(\frac{\lambda}{\lambda+1}\right)^{k} \\
    &=1-\left(\frac{\lambda}{\lambda+1}\right)^{K+1}.
\end{split}
\end{equation}

\section{Proof of Theorem 1}
\label{appendix:upper bound with hoeffding}
We first introduce the General Hoeffding Inequality~\citep{hoeffding1994probability}, which is essential for bounding $\left\|\mathbf{S}\bm{\Upsilon}\right\|_{F}^{2}$.

\begin{lemma}
\label{hoeffding}
\rm{\textbf{(General Hoeffding Inequality~\citep{hoeffding1994probability})}} \emph{Suppose that the variables $X_{1}, \cdots, X_{n}$ are independent, and $X_i$ has mean $\mu_{i}$ and sub-Gaussian parameter $\sigma_{i}$. Then for all $t\geq0$, we have}
\begin{equation}
 \mathbb{P}\left[\sum_{i=1}^{n}\left(X_{i}-\mu_{i}\right) \geq t\right] \leq \exp \left\{-\frac{t^{2}}{2 \sum_{i=1}^{n} \sigma_{i}^{2}}\right\}.
\end{equation}
\end{lemma}

Now, we prove Theorem 1.
\begin{proof}[Proof of Theorem 1.]
For any entry $\left[\mathbf{S}\bm{\Upsilon}\right]_{ij}=\sum_{p=1}^{n}\left(\mathbf{S}\right)_{ip}\bm{\Upsilon}_{pj}$, where $\bm{\Upsilon}_{pj}$ is a sub-Gaussian variable with parameter $\sigma^{2}$. By the General Hoeffding inequality~\ref{hoeffding}, we have
\begin{equation}
\begin{split}
&\mathbb{P}\left(\left|\left[ \frac{1}{\lambda+1}\sum_{k=0}^{K}\left(\frac{\lambda}{\lambda+1}\widetilde{\bm{\mathcal{A}}}\right)^{k}\bm{\Upsilon}\right]_{ij}\right|\geq t\right)\\
&\leq 2\exp \left\{-\frac{nt^{2}}{2\tau\left(1-\left(\frac{\lambda}{\lambda+1}\right)^{K+1}\right)^2\sigma^2}\right\},
\end{split}
\end{equation}
where $\tau = \max_i \tau_i$ and $ \tau_i = {n\sum_{j=1}^{n}\left[\mathbf{S}\right]_{ij}^2}\Bigg/{\left(1-\left(\frac{\lambda}{\lambda+1}\right)^{K+1}\right)^2}$. 

Applying union bound~\citep{vershynin2010introduction} to all possible pairs of $i \in [n]$, $j \in [n]$, we get
\begin{equation}
\begin{split}
 \mathbb{P}\left(\left\|\mathbf{S}\bm{\Upsilon}\right\|_{\infty, \infty}\geq t\right) &\leq \sum_{i,j}\mathbb{P}\left(\left[\mathbf{S}\bm{\Upsilon}\right]_{ij}\geq t\right)\\
 &\leq 2n^2\exp \left\{-\frac{nt^{2}}{2\tau\left(1-\left(\frac{\lambda}{\lambda+1}\right)^{K+1}\right)^2\sigma^2}\right\}. 
 \end{split}
\end{equation}
Applying union bound again, we have
\begin{equation}
\begin{split}
   & \mathbb{P}\left(\left\|\mathbf{S}\bm{\Upsilon}\right\|_{F}^{2}\geq t\right) \\ \leq & \sum_{i,j}\mathbb{P}\left(\left\|\mathbf{S}\bm{\Upsilon}\right\|_{\infty, \infty}\geq \sqrt{t}\right)\\
   \leq & 2n^4\exp \left\{-\frac{nt}{2\tau\left(1-\left(\frac{\lambda}{\lambda+1}\right)^{K+1}\right)^2\sigma^2}\right\}.
\end{split}
\end{equation}
Choose $t=2\tau\left(1-\left(\frac{\lambda}{\lambda+1}\right)^{K+1}\right)^2\left(4\log n+\log{2d}\right)/n$ and with probability $1-1/d$, we have 
\begin{equation}
    \left\|\mathbf{S}\bm{\Upsilon}\right\|_{F}^{2}\leq \frac{2\tau\left(1-\left(\frac{\lambda}{\lambda+1}\right)^{K+1}\right)^2\sigma^2\left(4\log n+\log{2d}\right)}{n},
\end{equation}
which completes the proof.
\end{proof}

\section{Related Work}
\label{appendix:related}
\paragraph{Adversarial Attacks on Graph Structure.} Several defense methods have recently been proposed to counter graph structural attacks~\citep{zhu2019robust,wang2022improving,chang2022adversarial,xie2023adversarially,gosch2023adversarial}.~\citet{wu2019adversarial} introduce GNN-Jaccard, which works with the graph's adjacency matrix to identify fake edges. RobustGCN~\citep{zugner2019robustgcn} employs Gaussian distributions in hidden layers to mitigate the effects of attacks. GNN-SVD~\citep{entezari2020all} reduces the rank of the adjacency matrix to counteract NETTACK, which targets high-rank singular components in graph data. Lastly, Pro-GNN~\citep{jin2020graph} presents an approach that jointly generates a new structural graph and a robust GNN model from the perturbed graph. By utilizing the additional term in our proposed GADC (IV), we can reconstruct the informative adjacency matrix, thus restoring efficient information flow and defending against adversarial structure attacks.

\paragraph{Implicit Denoising in GNNs.} Existing graph denoising works primarily focus on graph smoothing techniques~\citep{chen2014signal,wang2021dissecting,zhou2021graph,chen2022optimization}. It is well-established that GNNs can increase node feature smoothness through neighbor information aggregation, counteracting the influence of noisy features in the GNN's output. Some recent GNN models, such as GLP~\citep{li2019label}, S$^2$GC~\citep{zhu2021simple}, and IRLS~\citep{yang2021graph}, are derived from the perspective of signal denoising. Additionally, Ma et al.~\citep{ma2021unified} connect signal denoising to popular GNNs by considering the message passing scheme as a process of solving the GSD problem.

\section{Datasets Details}
\label{appendix:dataset}
\begin{table}[!hbtp]
\caption{Datasets statistics}
\label{tab:stat}
\begin{center}
\resizebox{\linewidth}{!}{
\setlength{\tabcolsep}{30pt}
\begin{tabular}{lcccc}
\toprule
  \multicolumn{1}{l}{\textbf{Dataset}}  & \# Nodes    & \# Edges   & \# Features  & \# Classes \\
\midrule
Cora  &  2708   & 5429 & 1433 & 7       \\
Citeseer    & 3327 & 4732 & 3703 & 6 \\
Pubmed   & 19717 & 44338 & 500 & 3   \\
Cornell   & 183 & 295 & 1703 & 5   \\
Texas & 183 & 309 & 1703 & 5 \\
Wisconsin & 251 & 499 & 1703 & 5 \\
Actor & 7600 & 33544 & 931 & 5 \\
Coauthor-CS & 18333 & 81894 & 6805 & 15 \\
Coauthor-Phy & 34493 & 247962 & 8415 & 5 \\
ogbn-products & 2449029 & 61859140 & 100 & 42\\
\bottomrule
\end{tabular}
}
\end{center}
\end{table}

\section{Hyperparameter Details}
\label{appendix:hyperparameter}
We provide details about the hyparatemeters of GADC in Table~\ref{tab:citation_hyper_attack}, ~\ref{tab:citation_hyper}, \ref{tab:co-author_hyper}, \ref{tab:products_hyper}, and \ref{tab:citation_hyper_flip}.

\begin{table}[htbp]
\setlength{\tabcolsep}{0.6mm}
\caption{The hyper-parameters for GADC (IV) on three citation datasets for defense evaluation against non-adaptive graph structure attacks.}
\label{tab:citation_hyper_attack}
    \centering
    \begin{tabular}{l|c|c|c|c|c|c|c|c|c|c}
    \toprule
    Model & dataset & runs  & lr & epochs & wight decay & hidden & dropout & $K$ & $\lambda$ & perturbation rate\\
    \midrule
    GADC (IV) & Cora & 10 & 0.02 & 100 & 1e-5 & 32 & 0.5 & 6 & 1 & 0.25 \\
    GADC (IV) & Cora & 10 & 0.02 & 100 & 1e-5 & 32 & 0.5 & 3 & 1 & 0.5 \\
    GADC (IV) & Cora & 10 & 0.02 & 100 & 1e-5 & 32 & 0.5 & 1 & 1 & 0.75 \\
    GADC (IV) & Citeseer & 10 & 0.02 & 100 & 1e-5 & 32 & 0.5 & 6 & 1 & 0.25 \\
    GADC (IV) & Citeseer & 10 & 0.02 & 100 & 1e-5 & 32 & 0.5 & 3 & 1 & 0.5 \\
    GADC (IV) & Citeseer & 10 & 0.02 & 100 & 1e-5 & 32 & 0.5 & 1 & 1 & 0.75 \\
    GADC (IV) & Pubmed & 10 & 0.02 & 200 & 1e-5 & 32 & 0.5 & 2 & 1 & 0.25 \\
    GADC (IV) & Pubmed & 10 & 0.02 & 200 & 1e-5 & 32 & 0.5 & 1 & 1 & 0.5 \\
    GADC (IV) & Pubmed & 10 & 0.02 & 200 & 1e-4 & 32 & 0.5 & 1 & 1 & 0.75 \\
    \bottomrule
    \end{tabular}
\end{table}

\begin{table}[htbp]
\setlength{\tabcolsep}{0.6mm}
\caption{The hyper-parameters for GADC (II) on three citation datasets for denoising evaluation against feature Gaussian noise.}
\label{tab:citation_hyper}
    \centering
    \begin{tabular}{l|c|c|c|c|c|c|c|c|c|c}
    \toprule
    Model & dataset & runs  & lr & epochs & wight decay & hidden & dropout & $K$ & $\lambda$ & $\varepsilon$\\
    \midrule
    GADC (II) & Cora & 100 & 0.2 & 100 & 1e-5 & 0 & 0 & 16 & 32 & 1 \\
    GADC (II) & Citeseer & 100 & 0.2 & 100 & 1e-5 & 0 & 0 & 16 & 32 & 1 \\
    GADC (II) & Pubmed & 100 & 0.2 & 100 & 1e-5 & 0 & 0 & 16 & 32 & 1 \\
    \bottomrule
    \end{tabular}
\end{table}

\begin{table}[t]
\setlength{\tabcolsep}{0.6mm}
\caption{The hyper-parameters for GADC (II) on two co-author datasets for denoising evaluation against feature Gaussian noise.}
\label{tab:co-author_hyper}
    \centering
    \begin{tabular}{l|c|c|c|c|c|c|c|c|c|c|c}
    \toprule
    Model & dataset & noise level & runs  & lr & epochs & wight decay & hidden & dropout & $K$ & $\lambda$ & $\varepsilon$ \\
    \midrule
    GADC (II) & Coauthor-CS & 0.1 & 10 & 0.2 & 1000 & 1e-7 & 0 & 0 & 16 & 1 & 1  \\
    GADC (II) & Coauthor-CS & 1 & 10 & 0.2 & 1000 & 1e-7 & 0 & 0 & 16 & 128 & 1  \\
    GADC (II) & Coauthor-Phy & 0.1 & 10 & 0.2 & 1000 & 1e-7 & 0 & 0 & 16 & 1 & 1   \\
    GADC (II) & Coauthor-Phy & 1 & 10 & 0.2 & 1000 & 1e-7 & 0 & 0 & 16 & 128 & 1  \\
    \bottomrule
    \end{tabular}
\end{table}

\begin{table}[t]
\setlength{\tabcolsep}{0.6mm}
\caption{The hyper-parameters for GADC (III) on ogbn-products dataset for denoising evaluation against feature Gaussian noise.}
\label{tab:products_hyper}
    \centering
    \begin{tabular}{l|c|c|c|c|c|c|c|c|c|c|c}
    \toprule
    Model  & noise level & runs  & lr & epochs & hidden & dropout & $K$ & $\lambda$ & $\varepsilon$ & layers & +MLP \\
    \midrule
    GADC (III) & 0.1 & 10 & 0.01 & 300 & 256 & 0.5 & 128 & 32 & 1e-2  & 3 & True \\
    GADC (III) & 1 & 10 & 0.01 & 300 & 256 & 0.5 & 128 & 256 & 1e-2  & 3 & True\\
    \bottomrule
    \end{tabular}
\end{table}

\begin{table}[htbp]
\setlength{\tabcolsep}{0.6mm}
\caption{The hyper-parameters for GADC (II) on three citation datasets for denoising evaluation against feature flip noise.}
\label{tab:citation_hyper_flip}
    \centering
    \begin{tabular}{l|c|c|c|c|c|c|c|c|c|c|c}
    \toprule
    Model & dataset & flip probability & runs  & lr & epochs & wight decay & hidden & dropout & $K$ & $\lambda$ & $\varepsilon$\\
    \midrule
    GADC (II)& Cora  & 0.1 & 100 & 0.2 & 100 & 1e-5 & 0 & 0 & 32 & 64 & 1e-5 \\
    GADC (II)& Cora  & 0.2 & 100 & 0.2 & 100 & 1e-5 & 0 & 0 & 32 & 64 & 1e-5 \\
    GADC (II)& Cora  & 0.4 & 100 & 0.2 & 100 & 1e-5 & 0 & 0 & 32 & 64 & 1e-1 \\
    GADC (II)& Citeseer  & 0.1 & 100 & 0.2 & 100 & 1e-5 & 0 & 0 & 32 & 64 & 1e-5 \\
    GADC (II)& Citeseer  & 0.2 & 100 & 0.2 & 100 & 1e-5 & 0 & 0 & 32 & 64 & 1e-5 \\
    GADC (II)& Citeseer  & 0.4 & 100 & 0.2 & 100 & 1e-5 & 0 & 0 & 32 & 64 & 1e-5 \\
    GADC (II)& Pubmed  & 0.1 & 100 & 0.2 & 100 & 1e-5 & 0 & 0 & 32 & 64 & 1e-1 \\
    GADC (II)& Pubmed  & 0.2 & 100 & 0.2 & 100 & 1e-5 & 0 & 0 & 32 & 64 & 1e-1 \\
    GADC (II)& Pubmed  & 0.4 & 100 & 0.2 & 100 & 1e-5 & 0 & 0 & 32 & 64 & 1e-1 \\
    \bottomrule
    \end{tabular}
\end{table}

\section{Denoising Performance against Flipping Perturbations}
We provide results in Table~\ref{table:flip}.
\label{appendix:result}

\begin{table*}[!htbp]
\caption{Denoising performance over 100 runs against flipping perturbation}
\begin{center}
    \setlength{\tabcolsep}{3mm}
        \scalebox{0.86}{\begin{tabular}{ccccccccccc}
        \toprule
        \multirow{2}{*}{Flipping probability} & \multicolumn{3}{c}{Cora} & \multicolumn{3}{c}{Citeseer} & \multicolumn{3}{c}{Pubmed}   \\
        \cmidrule(r){2-4} \cmidrule(r){5-7}  \cmidrule(r){8-10} 
        & 0.1 &  0.2 & 0.4 & 0.1 &  0.2 & 0.4 & 0.1 &  0.2 & 0.4 \\
        \midrule
        MLP & 21.2\scriptsize{$\pm$7.3} & 21.1\scriptsize{$\pm$8.0} & 23.3\scriptsize{$\pm$8.0} & 19.3\scriptsize{$\pm$3.3}  & 18.9\scriptsize{$\pm$2.8} & 18.9\scriptsize{$\pm$2.7} & 38.0\scriptsize{$\pm$6.3} & 39.0\scriptsize{$\pm$4.7} & 40.6\scriptsize{$\pm$2.8} \\
        GCN & 22.9\scriptsize{$\pm$13.6} & 19.0\scriptsize{$\pm$9.4} & 19.0\scriptsize{$\pm$9.3} & 18.6\scriptsize{$\pm$3.4} & 18.6\scriptsize{$\pm$3.1} & 18.5\scriptsize{$\pm$3.2} & 37.8\scriptsize{$\pm$6.6} & 38.1\scriptsize{$\pm$7.1} & 37.6\scriptsize{$\pm$8.0} \\
        GAT & 70.1\scriptsize{$\pm$1.5} & 65.6\scriptsize{$\pm$1.5} & 60.0\scriptsize{$\pm$3.2} & 45.3\scriptsize{$\pm$2.9} & 39.3\scriptsize{$\pm$3.4} & 26.0\scriptsize{$\pm$5.1} & 43.3\scriptsize{$\pm$2.7} & 49.5\scriptsize{$\pm$3.2} & 60.0\scriptsize{$\pm$4.1} \\
        APPNP & 75.6\scriptsize{$\pm$1.2} & 69.8\scriptsize{$\pm$1.5} & 65.3\scriptsize{$\pm$1.6} & \textbf{56.5\scriptsize{$\pm$2.2}} & 49.1\scriptsize{$\pm$1.8} & 42.8\scriptsize{$\pm$3.0} & 43.4\scriptsize{$\pm$2.7} & 52.3\scriptsize{$\pm$1.8} & 64.4\scriptsize{$\pm$1.7} \\
        GLP & 32.3\scriptsize{$\pm$0.8} & 30.8\scriptsize{$\pm$4.0} & 29.0\scriptsize{$\pm$6.4} & 19.7\scriptsize{$\pm$2.7} & 18.9\scriptsize{$\pm$2.4} & 18.8\scriptsize{$\pm$2.3} & 42.1\scriptsize{$\pm$2.0} & 41.5\scriptsize{$\pm$1.8} & 40.7\scriptsize{$\pm$0.1} \\
        S$^2$GC & 75.0\scriptsize{$\pm$1.6} & 71.5\scriptsize{$\pm$2.0} & 63.8\scriptsize{$\pm$4.4} & 49.9\scriptsize{$\pm$3.9} & 46.4\scriptsize{$\pm$3.2} & 43.4\scriptsize{$\pm$2.9} & 50.4\scriptsize{$\pm$2.2} & 60.2\scriptsize{$\pm$1.9} & 69.3\scriptsize{$\pm$1.6} \\
        IRLS & 66.4\scriptsize{$\pm$2.0} & 61.0\scriptsize{$\pm$1.9} & 54.7\scriptsize{$\pm$2.5} & 50.3\scriptsize{$\pm$2.7} & 45.9\scriptsize{$\pm$2.1} & 43.8\scriptsize{$\pm$1.7} & 51.4\scriptsize{$\pm$4.1} & 60.0\scriptsize{$\pm$3.8} & 69.0\scriptsize{$\pm$2.4} \\
        GADC (II) & \textbf{77.6\scriptsize{$\pm$1.2}} & \textbf{75.2\scriptsize{$\pm$1.4}} & \textbf{72.8\scriptsize{$\pm$1.4}} & {55.0\scriptsize{$\pm$3.0}} & \textbf{51.8\scriptsize{$\pm$2.4}} & \textbf{48.7\scriptsize{$\pm$2.4}} & \textbf{54.3\scriptsize{$\pm$2.0}} & \textbf{63.9\scriptsize{$\pm$1.6}} & \textbf{71.6\scriptsize{$\pm$1.1}} \\
        \bottomrule
        \end{tabular}}
        \end{center}
        \label{table:flip}
\end{table*}

\section{Ablation Study on Heterophilic Graphs}
\label{appendix:heterophilic}
In this section, we show the improvement of our proposed GADC (I) on heterophilic graphs with the additional term in the modified transition matrix through a series of ablation studies.

\paragraph{Datasets.} We use four datasets for fully supervised node classification on heterophilic graphs: Cornell, Texas, Wisconsin, and Actor. For each dataset, we randomly split nodes into 60\%, 20\%, and 20\% for training, validation, and testing, as suggested in~\citet{pei2020geom}.

\paragraph{Setting and Results.} We evaluate the performance of the additional term on heterophilic graphs by setting $\varepsilon$ to a series of values. We also add a 2-layer MLP with 64 hidden units after the feature aggregation of GADC (I). For GADC (I), we set $\lambda$ to 1 and $K$ to 16. We conduct experiments 100 times and report the mean test accuracy for the node classification task. Note that in each repeated run, we use a different dataset split. From the results summarized in Table~\ref{table:heterophily}, it is evident that incorporating the additional term ($\varepsilon \in \{1.0, 2.0\}$) can lead to improved performance compared to the scenario where it is disabled ($\varepsilon=0$). Furthermore, we notice that when $\varepsilon$ is set to 1.0, the performance improvement is maximized. This observation suggests that these specific perturbation levels could be optimal for the context of our study.
\begin{table}[!htbp]
\setlength{\tabcolsep}{3mm}
\centering
\caption{Test accuracy with $100$ runs of GADC (I) on heterophilic graphs. For these datasets, we use $\varepsilon$ of 0, 1.0, 2.0, 3.0, and 4.0, respectively. }
{\scriptsize
\begin{tabular}{l|c|c|c|c|c}
\hline
\toprule
$\varepsilon$ & 0 & 1.0 & 2.0 & 3.0 & 4.0  \\
\midrule
Cornell & 74.8 $\pm$ 7.0 & \textbf{76.9 $\pm$ 7.6} & 76.4 {$\pm$ 7.6} & 69.0 {$\pm$ 7.9} & 57.5 {$\pm$ 8.6}  \\
Texas & 74.9 {$\pm$ 6.1} & \textbf{77.8 {$\pm$ 6.9}} & 76.7 {$\pm$ 7.5} &  64.6 {$\pm$ 7.9} & 60.8 {$\pm$ 8.0} \\
Wisconsin & 73.2 {$\pm$ 5.9} & \textbf{78.2 {$\pm$ 6.5}} & 78.0 {$\pm$ 6.7} & 64.8 {$\pm$ 6.5} & 53.1 {$\pm$ 8.2} \\
Actor & 34.35 {$\pm$ 1.35} & \textbf{34.51 {$\pm$ 1.41}} & 25.42 {$\pm$ 1.07} & 25.30 {$\pm$ 1.06} & 25.16 {$\pm$ 1.05} \\
\bottomrule
\end{tabular}
}
\label{table:heterophily}
\end{table}

\section{Illustration of Higher-order Graph Connectivity Factor on Various Graphs}
\label{appendix:illustration}
\begin{figure*}[htbp] 
\centering 
\includegraphics[width=0.8\textwidth]{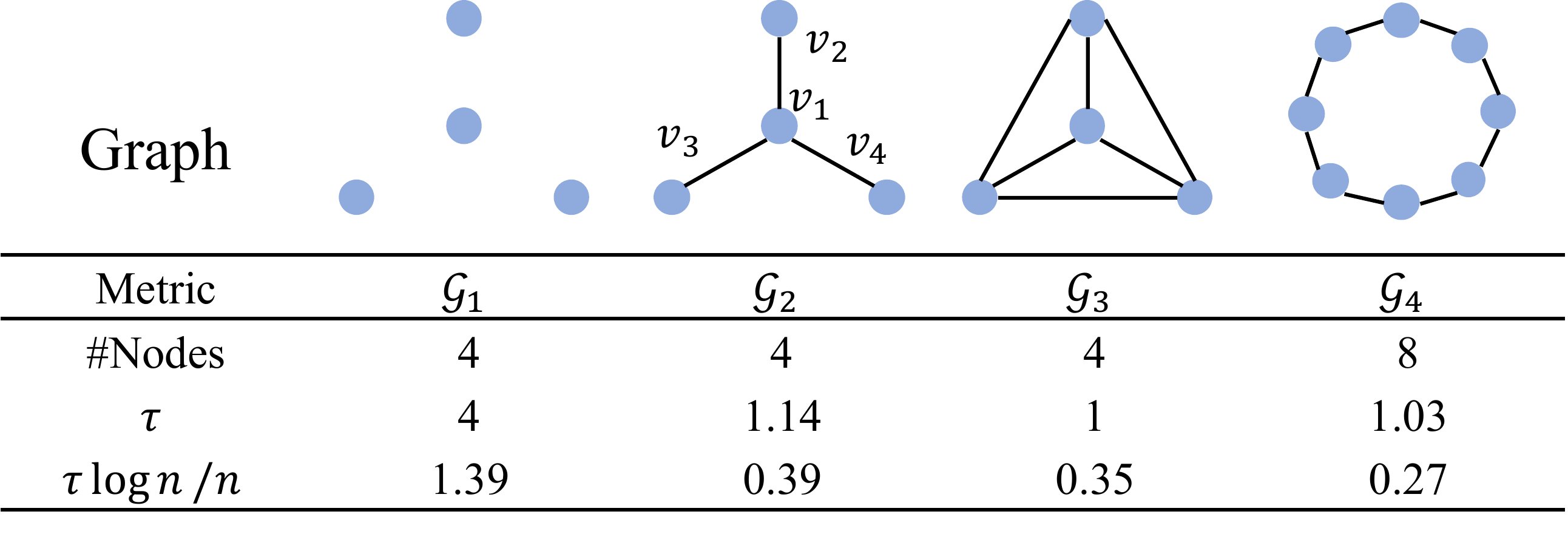} 
\caption{An illustration of $\tau$ on various graph structures. $\mathcal{G}_1$: nodes are isolated; $\mathcal{G}_2$: a star graph with 4 nodes; $\mathcal{G}_3$: a complete graph with 4 nodes. For computing $\tau$, we set $\lambda$ and $K$ as 32. $\tau$ has a smaller value if the graph has good connectivity.} 
\label{fig:graph structure} 
\end{figure*}
In this case study, we present four illustrative samples in Figure~\ref{fig:graph structure}: $G_1$, $G_2$, $G_3$, and $G_4$.

Graphs $G_1$, $G_2$, and $G_3$ have the same number of nodes. However, the nodes in $G_1$ are isolated, $G_2$ consists of a single connected component with a central node $v_1$, and $G_3$ is a complete graph where each node is connected to every other node.
Additionally, we include a larger graph, $G_4$, to analyze the influence of graph size on the denoising effect.
From Figure~\ref{fig:graph structure}, we can derive the following insights:
\begin{itemize}
    \item No Denoising on Isolated Graph ($G_1$): Since the nodes in $G_1$ are isolated, there is no denoising effect, as indicated by the large value of $\tau\log n /n$.

    \item Best Denoising on Complete Graph ($G_3$): The complete graph $G_3$ exhibits the best denoising effect among the graphs of the same size. This is because the elements in each row are uniformly distributed, leading to the lowest possible value of $\tau$.

    \item Central Node Impact on Connected Graph ($G_2$): Although $G_2$ has only one connected component, the presence of a central node $v_1$ creates an imbalance in the value of elements in each row. Consequently, $\tau$ tends to be larger compared to $G_3$.

    \item Denoising on Decentralized Larger Graph ($G_4$): The decentralized structure of $G_4$ also results in a smaller value of $\tau$, indicating a good denoising effect.

    \item Influence of Graph Size: Larger graphs, such as $G_4$, tend to have a better denoising effect due to their size.
\end{itemize}
By analyzing these graphs, we get a deeper understanding of how graph structure and size influence the denoising effect.

\section{Reproducibility}
We use Pytorch~\citep{paszke2019pytorch} and PyG~\citep{fey2019fast} to implement GADC. All the experiments in this work are conducted on a single NVIDIA Tesla A100 with 80GB memory size. The software that we use for experiments are Python 3.6.8, pytorch 1.9.0, pytorch-scatter 2.0.9, pytorch-sparse 0.6.12, pyg 2.0.3, ogb 1.3.4, numpy 1.19.5, torchvision 0.10.0, and CUDA 11.1.

\section{Work Statement}
This is an extended work based on our manuscript~\citep{liu2022powerful} that appeared in 2022 NeurIPS GFrontiers. Songtao Liu independently extends \citet{liu2022powerful} and rewrites the initial versions of this work.

\end{document}